\newif\ifdrafting
\draftingfalse 

\documentclass[10pt,twocolumn,twoside]{IEEEtran}

\usepackage{url}
\usepackage{cite}
\usepackage{times}
\usepackage{epsfig}
\usepackage{epstopdf}
\usepackage{graphicx}
\usepackage{amsmath}
\usepackage{amssymb}
\usepackage{color}
\usepackage{tikz}
\usetikzlibrary{positioning}
\usepackage{subcaption}

\newtheorem{remark}{Remark}
\newtheorem{theorem}{Theorem}
\newtheorem{lemma}{Lemma}
\newtheorem{definition}{Definition}
\newtheorem{property}{Property}
\newtheorem{corollary}{Corollary}

\definecolor{darkgreen}{RGB}{0,127,0}
\definecolor{darkblue}{RGB}{0,0,127}
\definecolor{darkred}{RGB}{127,0,0}
\definecolor{darkmagenta}{RGB}{127,0,127}
\definecolor{darkcyan}{RGB}{0,127,127}

\allowdisplaybreaks

\ifdrafting
\newcommand{\AM} [1] {\textcolor{darkblue}{[AM: #1]}} 
\newcommand{\YM}[1] {\textcolor{darkgreen}{[YM: #1]}} 
\newcommand{\Todo} [1] {\textcolor{darkmagenta}{\bf [Todo: #1]}}
\else
\newcommand{\AM} [1] {}
\newcommand{\YM} [1] {}
\newcommand{\Todo} [1] {}
\fi

\newcommand{\RC}[1]{#1}    

\newcommand{\rpm}{\raisebox{.2ex}{$\scriptstyle\pm$}}
\DeclareMathOperator*{\argmin}{arg\,min}

\begin{document}
\epstopdfsetup{outdir=./}

\thispagestyle{plain} \pagestyle{plain} \date{1st April 2017}
\title{Path Planning for Minimizing the \\
Expected Cost \RC{until} Success}
\author{Arjun~Muralidharan~and~Yasamin~Mostofi
\thanks{This work is supported in part by NSF RI award 1619376 and NSF CCSS award 1611254.}
\thanks{The authors are with the Department of Electrical and Computer Engineering,
University of California Santa Barbara, Santa Barbara, CA 93106, USA email:
$\{$arjunm, ymostofi$\}$@ece.ucsb.edu.}}
\maketitle

\begin{abstract}
Consider a general path planning problem of a robot on a graph with edge costs, and where each node has a Boolean value of success or failure (with respect to some task) with a given probability.
The objective is to plan a path for the robot on the graph that minimizes the expected cost \RC{until} success.
In this paper, it is our goal to bring a foundational understanding to this problem.
We start by showing how this problem can be optimally solved by formulating it as an infinite horizon Markov Decision Process, but with an exponential space complexity.
We then formally prove its NP-hardness.
To address the space complexity, we then propose a path planner, using a game-theoretic framework, that asymptotically gets arbitrarily close to the optimal solution.
Moreover, we also propose two fast and non-myopic path planners.
To show the performance of our framework, we do extensive simulations for two scenarios: a rover on Mars searching for an object for scientific studies, and a robot looking for a connected spot to a remote station (with real data from downtown San Francisco).
Our numerical results show a considerable performance improvement over existing state-of-the-art approaches.
\end{abstract}

\section{Introduction}\label{sec:intro}

Consider the scenario of a rover on mars looking for an object of interest, for instance a sample of water, for scientific studies.
Based on prior information, it has an estimate of the likelihood of finding such an object at any particular location.
The goal in such a scenario would be to locate one such object with a minimum expected cost.
\RC{Note that there may be multiple such objects in the environment, and that we only care about the expected cost until the first such object is found.}
In this paper, we tackle such a problem by posing it as a graph-theoretic path planning problem where there is a probability of success in finding an object associated with each node.
The goal is then to \emph{plan a path} through the graph that would \emph{minimize the expected cost} \RC{until} an object of interest is successfully found.
Several other problems of interest also fall into this formulation.
For instance, the scenario of a robot looking for a location connected to a remote station can be posed in this setting \cite{muralidharan2018pconn},
\RC{where a connected spot is one where the signal reception quality from/to the remote node/station is high enough to facilitate the communication.}
The robot can typically have a probabilistic assessment of connectivity all over the workspace, without a need to visit the entire space \cite{malmirchegini2012spatial}.
Then, it is interested in planning a path that gets it to a connected spot while minimizing the total energy consumption.
\RC{Success in this example corresponds to the robot getting connected to the remote station.}
Another scenario would be that of astronomers searching for a habitable exoplanet.
Researchers have characterized the probability of finding exoplanets in different parts of space \cite{molaverdikhani2009mapping}.
However, repositioning satellites to target and image different celestial objects is costly and consumes fuel.
Thus, a problem of interest in this context, is to find an exoplanet while minimizing the expected fuel consumption, based on the prior probabilities.
Finally, consider a human-robot collaboration scenario, where an office robot needs help from a human, for instance in operating an elevator \cite{rosenthal2012someone}.
If the robot has an estimate of different people's willingness to help, perhaps from past observations, it can then plan its trajectory to minimize its energy consumption \RC{until} it finds help.
Fig. \ref{fig:scenarios} showcases a sample of these possible applications.

\begin{figure}
    \centering
    \includegraphics[width=1\linewidth]{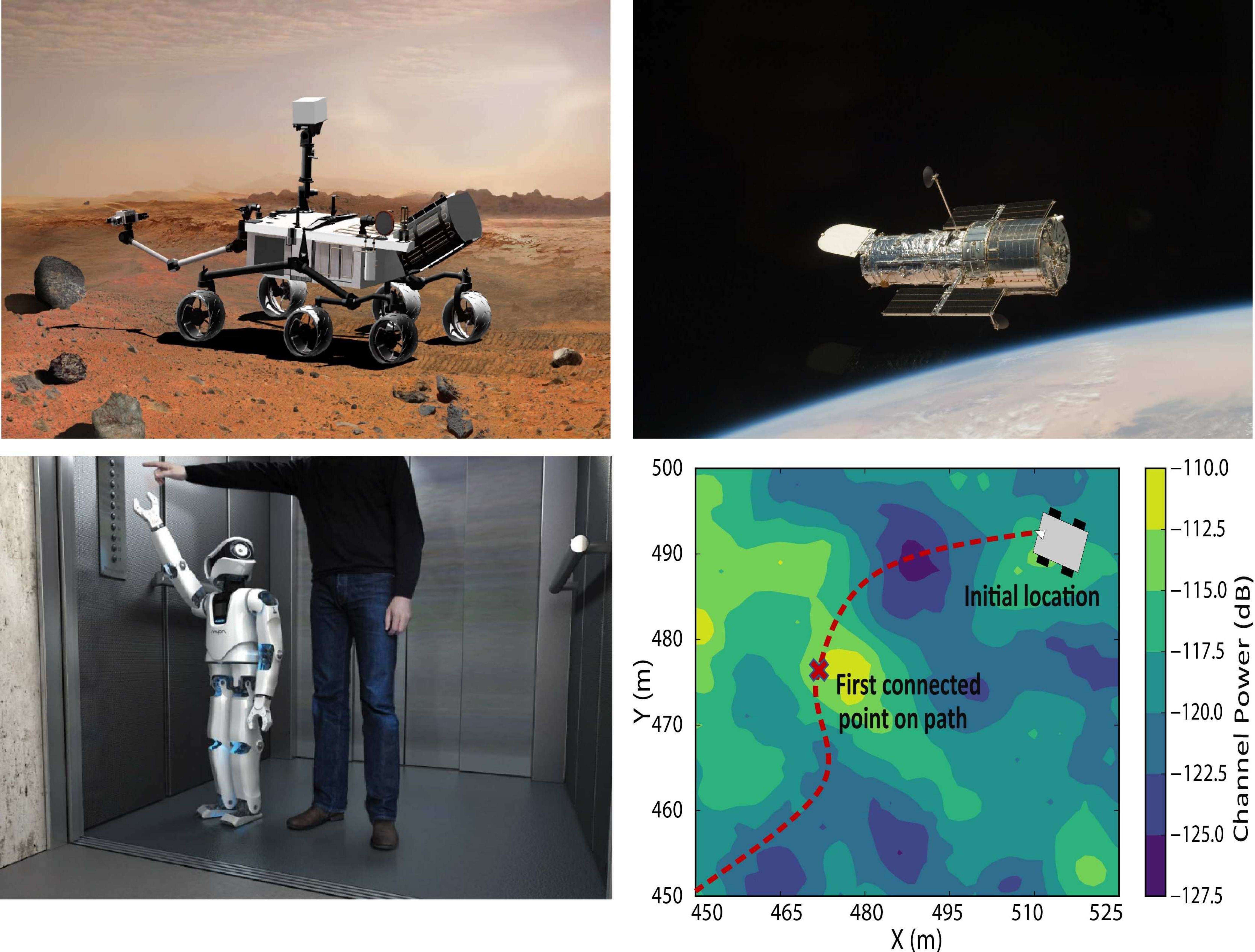}
    \vspace{-0.2in}
\caption{\small Possible applications of the problem of interest: (top left) path planning for a rover, (top right) imaging of celestial objects, (bottom left) human-robot collaboration and (bottom right) path planning to find a connected spot.
Image credit:(top left) and (top right) NASA, (bottom left) Noto: \protect\url{http://www.noto.design/}.}
    \label{fig:scenarios}
    \vspace{-0.3in}
\end{figure}

Optimal path planning for a robot has received considerable interest in the research community, and several algorithms have been proposed in the literature to tackle such problems, e.g., A*, RRT* \cite{karaman2010incremental,likhachev2008anytime}.
These works are concerned with planning a path for a robot, with a minimum cost, from an initial state to a predefined goal state.
However, this is different from our problem of interest in several aspects.
For instance, the cost metric is additive in these works, which does not apply to our setting due to its stochastic nature.
\RC{In the probabilistic traveling salesman problem \cite{jaillet1985probabilistic} and the probabilistic vehicle routing problem \cite{bertsimas1992vehicle}, each node is associated with a prior  probability of having a demand to be serviced, and the objective is to plan an a priori ordering of the nodes which minimizes the expected length of the tour.
A node is visited in a particular realization only if there is a demand to be serviced at it.
Thus, each realization has a different tour associated with it, and the expectation is computed over these tours, which is a fundamentally different problem than ours.}
Another area of active research is in path planning strategies for a robot searching for a target \cite{chung2012analysis,hollinger2009efficient,chung2011search}.
For instance, in \cite{chung2012analysis}, a mobile robot is tasked with locating a stationary target in minimum expected time.
In \cite{hollinger2009efficient}, there are multiple mobile robots and the objective is to find a moving target efficiently.
In general, these papers belong to a body of work known as optimal search theory where the objective is to find a \emph{single} hidden target based on an initial probability estimate, where the probabilities over the graph sum up to one \cite{bourgault2003optimal,chung2011search}.
\RC{The minimum latency problem \cite{blum1994minimum} is another problem related to search where the objective is to design a tour that minimizes the average wait time until a node is visited.}
In contrast, our setting is fundamentally different, and involves an \emph{unknown} number of targets where each node has a probability of containing a target ranging from $0$ to $1$.
Moreover, the objective is to plan a path that minimizes the expected cost to the \emph{first target} found.
This results in a different analysis and we utilize a different set of tools to tackle this problem.
Another related problem is that of satisficing search in the artificial intelligence literature which deals with planning a sequence of nodes to be searched \RC{until} the first satisfactory solution is found, which could be the proof of a theorem or a task to be solved \cite{simon1975optimal}.
The objective in this setting is to minimize the expected cost \RC{until} the first instance of success.
However, in this setting there is no cost associated with switching the search from one node to another.
To the best of the authors knowledge, the problem considered in this paper has not been explored before.

\textbf{Statement of contribution:}
In this paper, we start by showing that the problem of interest, i.e., minimizing the expected cost \RC{until} success, can be posed as an infinite horizon Markov Decision Process (MDP) and solved optimally, but with an exponential space complexity.
We then formally prove its NP-hardness.
To address the space complexity, we then propose an asymptotically $\epsilon$-suboptimal (i.e., within $\epsilon$ of the optimal solution value) path planner for this problem, using a game-theoretic framework.
We further show how it is possible to solve this problem very quickly by proposing two
sub-optimal but non-myopic approaches.
Our proposed approaches provide a variety of tools that can be suitable for applications with different needs.
A small part of this work has appeared in its conference version [1].
In [1], we only considered the specific scenario of a robot seeking connectivity and only discussed a single suboptimal non-myopic path planner.
This paper has a considerably more extensive analysis and results.

The rest of the paper is organized as follows.
In Section \ref{sec:problem_formulation}, we formally introduce the problem of interest and show how to optimally solve it by formulating it in an infinite horizon MDP framework as a stochastic shortest path (SSP) problem.
As we shall see, however, the state space requirement for this formulation is exponential in the number of nodes in the graph.
In Section \ref{sec:comp_complexity}, we formally prove our problem to be NP-hard, demonstrating that the exponential complexity result of the MDP formulation is not specific to it.
In Section \ref{sec:asmpt_near_opt_planner}, we propose an asymptotically $\epsilon$-suboptimal path planner and in Section \ref{sec:non_myopic_planners} we propose two suboptimal but non-myopic and fast path planners to tackle the problem.
Finally, in Section \ref{sec:numerical_results}, we confirm the efficiency of our approaches with numerical results in two different scenarios.

\section{Problem Formulation}\label{sec:problem_formulation}

In this section, we formally define the problem of interest,
which we refer to as the Min-Exp-Cost-Path problem.
We next show that we can find the optimal solution of Min-Exp-Cost-Path by formulating it as an infinite horizon MDP with an absorbing state, a formulation known in the stochastic dynamic programming literature as the \emph{stochastic shortest path} problem \cite{bertsekas1995dynamic}.
However, we show that this results in a state space requirement that is exponential in the number of nodes of the graph, implying that it is only feasible for small graphs and not scalable when increasing the size of the graph.

\subsection{Min-Exp-Cost-Path Problem}\label{subsec:min_exp_cost_path}
Consider an undirected \RC{connected} finite graph $\mathcal{G} = (\mathcal{V}, \mathcal{E})$, where $\mathcal{V}$ denotes the set of nodes and $\mathcal{E}$ denotes the set of edges.
Let $p_v \in [0,1]$ be the probability of success at node $v \in \mathcal{V}$
and let \RC{$l_{uv} > 0$} denote the cost of traversing edge $(u,v) \in \mathcal{E}$.
We assume that the success or failure of a node is independent of the success or failure of the other nodes in the graph.
Let $v_s \in \mathcal{V}$ denote the starting node.
The objective is to produce a path starting from node $v_s$ that \emph{minimizes the expected cost incurred \RC{until} success}.
In other words, the average cost \RC{until} success on the optimal path is smaller than the average cost on any other possible path on the graph.
Note that the robot may only traverse part of the entire path produced by its planning, as its planning is based on a probabilistic prior knowledge and success may occur at any node along the path.

For the expected cost \RC{until} success of a path to be well defined, the probability of failure after traversing the entire path must be $0$.
This implies that the final node of the path must be one where success is guaranteed, i.e., a $v$ such that $p_{v}=1$.
We call such a node a \emph{terminal} node and let $T=\{v \in \mathcal{V}: p_{v}=1\}$ denote the set of terminal nodes.
We assume that the set $T$ is non-empty in this subsection.
We refer to this as the \emph{Min-Exp-Cost-Path} problem.
Fig. \ref{fig:problem_setup} shows a toy example along with a feasible solution path.
In Section \ref{subsec:min_exp_cost_tour}, we will extend our discussion to the setting when the  the set $T$ is empty.
\AM{Point out that in reality only a part of the path may be traversed by the robot.
Depending on the realization different extents of the path may be covered.}

We next characterize the expected cost for paths where nodes are not revisited, i.e., simple paths, and then generalize it to all possible paths.
Let the path, $\mathcal{P} = (v_1, v_2, \cdots, v_m=v_t)$, be a sequence of $m$ nodes such that no node is revisited, i.e., $v_i \neq v_j,\; \forall i \neq j$,
and which ends at a terminal node $v_t \in T$.
Let $C(\mathcal{P},i)$ represent the expected cost of the path from node $\mathcal{P}[i]=v_i$ onward.
$C(\mathcal{P},1)$ is then given as 
\begin{align*}
C(\mathcal{P},1) & =  p_{v_1}\times 0 + (1-p_{v_1})p_{v_2}l_{v_1v_2} 
+ \cdots \\
& \;\;\;\;\;+ \Bigg[\prod_{j\leq m-1} (1-p_{v_j})\Bigg] p_{v_{m}}(l_{v_1v_2} + \cdots+ l_{v_{m-1}v_{m}})\\
& =  (1-p_{v_1})l_{v_1v_2} + (1-p_{v_1})(1-p_{v_2})l_{v_2v_3} + \cdots\\
& \;\;\;\;\;\;\;\;\;\;\;\;\;\;\;\;\;\;\;\;\;\;\;\;\;\;\;\;\;\;\;\;+ \Bigg[\prod_{j\leq m-1} (1-p_{v_j})\Bigg]l_{v_{m-1}v_{m}}\\
& =  \sum_{i=1}^{m-1} \left[\prod_{j\leq i} (1-p_{v_j})\right] l_{v_{i}v_{i+1}}.
\end{align*}
For a path which contains revisited nodes, the expected cost can then be given by
\begin{align*}
C(\mathcal{P},1) = & \sum_{i=1}^{m-1} \left[\prod_{j\leq i: v_j \neq v_k , \forall k<j} (1-p_{v_j})\right] l_{v_{i}v_{i+1}}\\
= & \sum_{e \in \mathcal{E}(\mathcal{P})} \left[\prod_{v \in \mathcal{V}(\mathcal{P}_{e})} (1-p_{v})\right] l_{e},
\end{align*}
where $\mathcal{E}(\mathcal{P})$ denotes the set of edges belonging to the path $\mathcal{P}$, and $\mathcal{V}(\mathcal{P}_{e})$ denotes the set of vertices encountered along $\mathcal{P}$ \RC{until} the edge $e \in \mathcal{E}(\mathcal{P})$.
Note that $\mathcal{C}(\mathcal{P},i)$ can be expressed recursively as 
\begin{align}\label{eq:min_exp_cost_recursion}
C(\mathcal{P},i) = \left\{\begin{array}{lll} (1-p_{v_i})\left(l_{v_iv_{i+1}} + C(\mathcal{P},i+1)\right),  \\ \;\;\;\;\;\;\;\;\;\;\;\;\;\;\;\;\;\;\;\;\;\;\;\;\;\;\;\;\;\;\;\;\;\;\;\;\;\; \text{if } v_i \neq v_k, \forall k<i \\
l_{v_iv_{i+1}} + C(\mathcal{P},i+1), \;\;\;\; \text{else}\end{array}\right. .
\end{align}
The Min-Exp-Cost-Path optimization can then be expressed as 
\begin{equation}\label{eq:min_exp_cost_path}
\begin{aligned}
& \underset{\mathcal{P}}{\text{minimize}} & & C(\mathcal{P},1)\\
& \text{subject to } & &  \mathcal{P} \text{ is a path of } \mathcal{G} \\
& & & \mathcal{P}[1] = v_s\\
& & & \mathcal{P}[\text{end}] \in T.
\end{aligned}
\end{equation}

\begin{figure}
\centering

\begin{tikzpicture}
  [scale=.6,auto=left,main node/.style={circle,fill=blue!20}]
  \node[main node] (n7) at (7,11) {$7$};
  \node[main node] (n6) at (1,10) {$6$};
  \node[main node] (n4) at (4,8)  {$4$};
  \node[main node] (n5) at (8,9)  {$5$};
  \node[main node] (n1) at (2,6) {$1$};
  \node[main node] (n3) at (9,6)  {$3$};
  \node[main node] (n2) at (6,5)  {$2$};
  
  \node (p1) [below=0.1cm of n1] {$p_{1}<1$};
  \node (p2) [below=0.1cm of n2] {$p_{2}<1$};
  \node (p7) [right=0.1cm of n7] {$p_{7}=1$};
  
  \path (n1) edge [->, thick, bend right =25] (n2);   
   \path (n1) edge [-, dashed] node [above] {$l_{12}$}(n2); 

  \path (n2) edge [->, thick, bend right =25] (n4);  
  \path (n4) edge [->, thick, bend left =25] (n6); 
  \path (n6) edge [->, thick, bend right =25] (n7); 
  
  \foreach \from/\to in {n6/n1,n6/n4,n4/n5,n4/n1,n2/n4,n2/n3,
                         n3/n5,n6/n7,n4/n7,n5/n7}
    \draw[dashed] (\from) -- (\to);

\end{tikzpicture}
\caption{\small{A toy example along with a feasible solution path starting from node $1$.}}
\label{fig:problem_setup}
\end{figure}
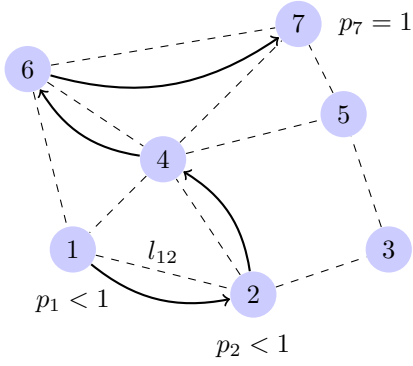

We next show how to optimally solve the Min-Exp-Cost-Path problem by formulating it as an infinite horizon MDP.

\subsection{Optimal Solution via MDP Formulation}\label{subsec:mdp_formulation}
The stochastic shortest path problem (SSP) \cite{bertsekas1995dynamic} is an infinite horizon MDP formulation, which is specified by a state space $S$, control/action constraint sets $A_s$ for $s \in S$, state transition probabilities $P_{ss'}(a_s)= \mathrm{P}\left(s_{k+1}=s'|s_{k}=s, a_k = a_s \right)$, an absorbing terminal state $s_{t} \in S$, and a cost function $g(s,a_s)$ for $s \in S$ and $a_s \in A_s$.
The goal is to obtain a policy that would lead to the terminal state $s_t$ with a probability $1$ and with a minimum expected cost.

We next show that the Min-Exp-Cost-Path problem formulation of (\ref{eq:min_exp_cost_path}) can be posed in an SSP formulation.
Utilizing the recursive expression of (\ref{eq:min_exp_cost_recursion}), we can see that the expected cost from a node, conditioned on the set of nodes already visited by the path can be expressed in terms of the expected cost from the neighboring node that the path visits next.
Thus, the optimal path from a node can be expressed in terms of the optimal path from the neighboring node that the path visits next, conditioned on the set of nodes already visited.
This motivates the use of a stochastic dynamic programming framework where a state is given by the current node as well as the set of nodes already visited.

More precisely, we formulate the SSP as follows.
Let $\mathcal{V}' = \mathcal{V}\setminus T$ be the set of non-terminal nodes in the graph.
A state of the MDP is given by $s=(v,H)$, where $v \in \mathcal{V}'$ is the current node and $H \subseteq \mathcal{V}'$ is the set of nodes already visited (keeping track of the history of the nodes visited), i.e., $u \in H$, if $u$ is visited.
The state space is then given by $S = \left\{(v,H):v \in \mathcal{V}', H \subseteq \mathcal{V}' \right\} \cup \{s_t\}$, where $s_t$ is the absorbing terminal state.
In this setting, the state $s_t$ denotes the state of success.
The actions/controls available at a given state is the set of neighbors of the current node,
i.e., $A_s = \{u \in \mathcal{V}: (v,u) \in \mathcal{E}\}$ for $s=(v,H)$.
The state transition probabilities are denoted by $P_{ss'}(u) = \mathrm{P}\left(s_{k+1}=s'|s_{k}=s, a_k = u \right)$ where $s,s' \in S$ and $u \in A_s$.
Then, for $s=(v,H)$ and $u \in A_s$, if $v \in H$ (i.e., $v$ is revisited), we have 
\begin{align*}
P_{ss'}(u) = \left\{\begin{array}{ll} 1, & \text{if } s'=f(u,H)\\
0, & \text{else}\end{array}\right. ,
\end{align*}
and if $v \notin H$, we have
\begin{align*}
P_{ss'}(u) = \left\{\begin{array}{lll} 1-p_{v}, & \text{if } s'=f(u,H)\\
p_{v}, & \text{if } s'=s_{t}\\
0, & \text{else}\end{array}\right.,
\end{align*}
where $
f(u,H) = \left\{\begin{array}{ll} (u,H\cup\{v\}), & \text{if } u \in \mathcal{V}'\\
s_{t}, & \text{if } u \in T \end{array}\right. $.
This implies that at node $v$, the robot will experience success with probability $p_v$ if $v$ has not been visited before, i.e., $v \notin H$.
The terminal state $s_t$ is absorbing, i.e., $P_{s_{t}s_{t}}(u) = 1, \forall u \in A_{s_t}$.
The cost $g(s,u)$ incurred when action/control $u \in A_{s}$ is taken in state $s\in S$ is given by $
g\left(s=(v,H), u\right) = \left\{\begin{array}{ll} (1-p_{v})l_{uv}, & \text{if } v \notin H\\
l_{uv}, & \text{if } v \in H\end{array} \right.$,
representing the expected cost incurred when going from $v$ to $u$ conditioned on the set of already visited nodes $H$.

The optimal (minimum expected) cost incurred from any state $s_1$ is then given by
\begin{align*}
J^{*}_{s_1} = \min_{\mu} \underset{\{s_{k}\}}{\mathbb{E}}\left[ \sum_{k=1}^{\infty} g(s_k,\mu_{s_k}) \right],
\end{align*}
where $\mu$ is a policy that prescribes what action to take/neighbor to choose at a given state, i.e., $\mu_s$ is the action to take at state $s$.
The policy $\mu$, specifies which node to move to next, i.e., if at state $s$, then $\mu_s$ denotes which node to go to next.
The objective is to find the optimal policy $\mu^{*}$ that would minimize the expected cost from any given state of the SSP formulation.
Given the optimal policy $\mu^{*}$, we can then extract the optimal solution path of (\ref{eq:min_exp_cost_path}).
Let $\left(s_1, \cdots , s_m=s_{t}\right)$ be the sequence of states such that $s_1 = (v_s,H_1=\{\})$ and $s_{k+1}=(v^{*}_{k+1},H_{k+1})$, $k=1,\cdots,m-2$, where $v^{*}_{k+1} = \mu^{*}(s_{k})$ and $H_{k+1}=H_{k} \cup \{v_{k}^{*}\}$.
This sequence must end at $s_m=s_t$ for some finite $m$, since the expected cost is not well defined otherwise.
The optimal path starting from node $v_s$ is then extracted from this solution as 
$
\mathcal{P}^{*} = (v_s, v^{*}_{2}, \cdots, v^{*}_{m})
$.

In the following Lemma, we show that the optimal solution can be characterized by the Bellman equation.
\begin{lemma}\label{lemma:mdp_bellman}
The optimal cost function $J^{*}$ is the unique solution of the Bellman equation:
\begin{align*}
J^{*}_{s} & = \min_{u \in A_{s}} \left[g(s,u) + \sum_{s'\in S\setminus\{s_t\}}P_{ss'}(u)J^{*}_{s'}\right],
\end{align*}
and the optimal policy $\mu^{*}$ is given by 
\begin{align*}
\mu^{*}_{s} = \argmin_{u \in A_{s}} \left[g(s,u) + \sum_{s'\in S\setminus\{s_t\}}P_{ss'}(u)J^{*}_{s'}\right],
\end{align*}
for all $s \in S\setminus\{s_t\}$.
\end{lemma}
\begin{proof}
\RC{
Let $J_{s}^{\mu}$ denote the cost of state $s$ for a policy $\mu$.
We first review the definition of a \emph{proper policy}.
A policy $\mu$ is said to be proper if, when using this policy, there is a positive probability that the terminal state will be reached after at most $|S|$ stages, regardless of the initial state \cite{bertsekas1995dynamic}.
We next show that the MDP formulation satisfies the following properties: 
1) there exists at least one proper policy, and 
2) for every improper policy $\mu$, there exists at least one state with cost $J_{s}^{\mu} = \infty$.
We know that there exists at least one proper policy since the policy corresponding to taking the shortest path to the nearest terminal node, irrespective of the history of nodes visited, is a proper policy.
Moreover, since $g(s,u)> 0 $ for all $s\neq s_t$, every cycle in the state space not including the destination has strictly positive cost.
This implies property $2$ is true.}
The proof is then provided in \cite{bertsekas1995dynamic}.
\end{proof}

The optimal solution can then be found by the value iteration method.
Given an initialization $J_s(0)$, for all $s \in S\setminus\{s_t\}$, value iteration produces the sequence:
\begin{align*}
J_s(k+1) = \min_{u \in A_{s}} \left[g(s,u) + \sum_{s'\in S\setminus\{s_t\}}P_{ss'}(u)J_{s'}(k)\right],
\end{align*}
for all $s \in S\setminus\{s_t\}$.
This sequence converges to the optimal cost $J^{*}_{s}$, for each $s \in S\setminus\{s_t\}$.

\RC{
\begin{lemma}\label{lemma:value_iteration_MDP}
When starting from $J_s(0) = \infty$ for all $s \in S\setminus \{s_t\}$, the value iteration method yields the optimal solution after at most $|S|=|\mathcal{V}'|\times 2^{|\mathcal{V}'|}+1$ iterations.
\end{lemma}
\begin{proof}
Let $\mu^{*}$ be the optimal policy.
Consider a directed graph with the states of the MDP as nodes, which has an edge $(s,s')$ if $P_{ss'}(\mu^{*}_s) > 0$.
We will first show that this graph is acyclic.
Note that a state $s=(v,H)$, where $v\notin H$, can never be revisited regardless of the policy used, since a transition from $s$ will occur either to $s_t$ or a state with $H = H\cup \{v\}$.
Then, any cycle in the directed graph corresponding to $\mu^{*}$ would only have states of the form $s = (v, H)$ with $v \in H$.
Moreover, any state $s=(v,H)$ in the cycle cannot have a transition to state $s_t$ since $v \in H$.
Thus, if there is a cycle, the cost of any state $s$ in the cycle will be $J_s^{\mu^{*}} = \infty$, which results in a contradiction.
The value iteration method converges in $|S|$ iterations when the graph corresponding to the optimal policy $\mu^{*}$ is acyclic \cite{bertsekas1995dynamic}.
\end{proof}
}

\RC{
\begin{remark}
Each stage of the value iteration process has a computational cost of $O(|\mathcal{E}|2^{|\mathcal{V}'|})$ since for each state $s=(v,H)$ there is an associated computational cost of $O(|A_v|)$.
Then, from Lemma \ref{lemma:value_iteration_MDP}, we can see that the overall computational cost of value iteration is $O(|\mathcal{V}'||\mathcal{E}|2^{2|\mathcal{V}'|})$, which is exponential in the number of nodes in the graph.
Note, however, that the brute force approach of enumerating all paths has a much larger computational cost of $O(|\mathcal{V}'|!)$.
\end{remark}
}

\RC{
The exponential space complexity prevents the stochastic shortest path formulation from providing a scalable solution for solving the problem for larger graphs.}
A general question then arises as to whether this high computational complexity result is a result of the Markov Decision Process formulation.
In other words, can we optimally solve the Min-Exp-Cost-Path problem with a low computational complexity using an alternate method?
We next show that the Min-Exp-Cost-Path problem is inherently computationally complex (NP-hard).

\section{Computational Complexity}\label{sec:comp_complexity}
In this section, we prove that Min-Exp-Cost-Path is NP-hard.
In order to do so, we first consider the extension of the Min-Exp-Cost-Path problem to the setting where there is no terminal node, which we refer to as the Min-Exp-Cost-Path-NT problem (Min-Exp-Cost-Path No Terminal node).
We prove that Min-Exp-Cost-Path-NT is NP-hard, a result we then utilize to prove that Min-Exp-Cost-Path is NP-hard.

Motivated by the negative space complexity result of our MDP formulation, we then discuss a setting where we restrict ourselves to the class of \emph{simple paths}, i.e., cycle free paths, and we refer to the minimum expected cost \RC{until} success problem in this setting as the Min-Exp-Cost-Simple-Path problem.
This serves as the setting for our path planning approaches of Section \ref{sec:asmpt_near_opt_planner} and \ref{sec:non_myopic_planners}.
Furthermore, we show that we can obtain a solution to the Min-Exp-Cost-Path problem from a solution of the Min-Exp-Cost-Simple-Path problem in an appropriately defined complete graph.

\subsection{Min-Exp-Cost-Path-NT Problem}\label{subsec:min_exp_cost_tour}
Consider the graph-theoretic setup of the Min-Exp-Cost-Path problem of Section \ref{subsec:min_exp_cost_path}.
In this subsection, we assume that there is no terminal node, i.e., the set $T = \{v \in \mathcal{V}: p_{v}=1\}$ is empty.
There is thus a finite probability of failure for any path in the graph and as a result the expected cost \RC{until} success is not well defined.
The expected cost of a path then includes the event of failure after traversing the entire path and its associated cost.
The objective in \emph{Min-Exp-Cost-Path-NT} is to obtain a path that visits all the vertices with a non-zero probability of success, i.e., $\{v \in \mathcal{V}: p_{v} > 0\}$, such that the expected cost is minimized.
This objective finds the minimum expected cost path among all paths that have a minimum probability of failure.
More formally, the objective for Min-Exp-Cost-Path-NT is given as 
\begin{equation}\label{eq:min_exp_cost_tour}
\begin{aligned}
& \underset{\mathcal{P}}{\text{minimize}} & & \sum_{e \in \mathcal{E}(\mathcal{P})} \left[\prod_{v \in \mathcal{V}(\mathcal{P}_{e})} (1-p_{v})\right] l_{e}\\
& \text{subject to } & &  \mathcal{P} \text{ is a path of } \mathcal{G} \\
& & & \mathcal{P}[1] = v_s\\
& & & \mathcal{V}(\mathcal{P}) = \{v \in \mathcal{V}: p_{v} >0\},
\end{aligned}
\end{equation}
where $\mathcal{V}(\mathcal{P})$ is the set of all vertices in path $\mathcal{P}$.

\begin{remark}
The Min-Exp-Cost-Path-NT problem is an important problem on its own (to address cases where no prior knowledge is available on nodes with $p_v=1$), even though we have primarily introduced it here to help prove that the Min-Exp-Cost-Path problem is NP-hard.
\end{remark}

\subsection{NP-hardness}
In order to establish that Min-Exp-Cost-Path is NP-hard, we first introduce the decision versions of Min-Exp-Cost-Path (MECPD) and Min-Exp-Cost-Path-NT (MECPNTD).

\begin{definition}[Min-Exp-Cost-Path Decision Problem]
Given a graph $\mathcal{G}=(\mathcal{V},\mathcal{E})$ with starting node $v_{s} \in \mathcal{V}$, edge weights $l_{e}$, $\forall e \in \mathcal{E}$, probability of success $p_{v} \in [0,1]$, $\forall v \in \mathcal{V}$, such that $T \neq \emptyset$, and budget $B_{\text{MECP}}$, does there exist a path $\mathcal{P}$ from $v_{s}$ such that the expected cost of the path $\mathcal{C}(\mathcal{P},1) \leq B_{\text{MECP}}$?
\end{definition}

\begin{definition}[Min-Exp-Cost-Path-NT Decision Problem]
Given a graph $\mathcal{G}=(\mathcal{V},\mathcal{E})$ with starting node $v_{s} \in \mathcal{V}$, edge weights $l_{e}, \forall e \in \mathcal{E}$, probability of success $p_{v} \in [0,1), \forall v \in \mathcal{V}$ and budget $B_{\text{MECPNT}}$,  does there exist a path $\mathcal{P}$ from $v_{s}$ that visits all nodes in $\{v \in \mathcal{V}: p_v>0\}$  such that  $\sum_{e \in \mathcal{E}(\mathcal{P})} \left[\prod_{v \in \mathcal{V}(\mathcal{P}_{e})} (1-p_{v})\right] l_{e} \leq B_{\text{MECPNT}}$?
\end{definition}

In the following Lemma, we first show that we can reduce MECPNTD to MECPD.
This implies that if we have a solver for MECPD, we can use it to solve MECPNTD as well.

\begin{lemma}\label{lemma:mecpntd_red_mecpd}
Min-Exp-Cost-Path-NT Decision problem reduces to Min-Exp-Cost-Path Decision problem.
\end{lemma}
\begin{proof}
Consider a general instance of MECPNTD with graph $\mathcal{G}=(\mathcal{V},\mathcal{E})$, starting node $v_{s} \in \mathcal{V}$, edge weights $l_{e}, \forall e \in \mathcal{E}$, probability of success $p_{v} \in [0,1), \forall v \in \mathcal{V}$, and budget $B_{\text{MECPNT}}$.
We create an instance of MECPD by introducing a new node $v_t$ into the graph with $p_{v_t}=1$.
We add edges of cost $l$ between $v_t$ and all the existing nodes of the graph.
We next show that if we choose a large enough value for $l$, then the Min-Exp-Cost-Path solution would visit all nodes in $\bar{\mathcal{V}} = \{v \in \mathcal{V}: p_v>0\}$ before moving to the terminal node $v_t$.
Let $l = 1.5D/\min_{v \in \bar{\mathcal{V}}}p_{v}$, where $D$ is the diameter of the graph.
Then, the Min-Exp-Cost-Path solution, which we denote by $\mathcal{P}^{*}$ must visit all nodes in $\bar{\mathcal{V}}$ before moving to node $v_t$.
We show this by contradiction.
Assume that this is not the case.
Since $\mathcal{P}^{*}$ has not visited all nodes in $\bar{\mathcal{V}}$, there exists a node $w \in \bar{\mathcal{V}}$ that does not belong to $\mathcal{P}^{*}$.
Let $\mathcal{Q}^{*}$ be the subpath of $\mathcal{P}^{*}$ that lies in the original graph $\mathcal{G}$ and let $u$ be the last node in $\mathcal{Q}^{*}$.
Consider the path $\mathcal{P}$ created by stitching together the path $\mathcal{Q}^{*}$, followed by the shortest path from $u$ to $w$ and then finally the terminal node $v_t$.
Let $p_{f} = \prod_{v \in \mathcal{V}(\mathcal{Q}^{*})} (1-p_{v})$ be the probability of failure after traversing path $\mathcal{Q}^{*}$.
The expected cost of path $\mathcal{P}$ then satisfies
\begin{align*}
\mathcal{C}(\mathcal{P},1) & \leq \sum_{e \in \mathcal{E}(\mathcal{Q}^{*})} \Bigg[\prod_{v \in \mathcal{V}(\mathcal{Q}^{*}_{e})} (1-p_{v})\Bigg] l_{e}  + \\
& \;\;\;\;\;\;\;\;\;\;\;\;\;\;\;\;\;\;\;\;\;\;\;\;\;\;\;\;\;\;\;\;\;\;\;\;\;\;\;\;\;\;\;\; p_{f} \left(l_{uw}^{\text{min}} + (1-p_{w})l\right)\\
& < \sum_{e \in \mathcal{E}(\mathcal{Q}^{*})} \Bigg[\prod_{v \in \mathcal{V}(\mathcal{Q}^{*}_{e})} (1-p_{v})\Bigg] l_{e}
 + p_{f}  l = \mathcal{C}(\mathcal{P}^{*},1),
\end{align*}
where $l_{uw}^{\text{min}}$ is the cost of the shortest path between $u$ and $w$.
We thus have a contradiction.

Thus, $\mathcal{Q}^{*}$ visits all the nodes in $\bar{\mathcal{V}}$.
Moreover, since $\mathcal{P}^{*}$ is a solution of Min-Exp-Cost-Path, we can see that $\mathcal{Q}^{*}$ must also be a solution of Min-Exp-Cost-Path-NT.
Thus, setting a budget of $B_{\text{MECP}} = B_{\text{MECPNT}} + p_{f}l$, where $p_{f} = \prod_{v \in \mathcal{V}(\mathcal{Q}^{*})} (1-p_{v}) = \prod_{v \in \bar{\mathcal{V}}} (1-p_{v})$, implies that the general instance of MECPNTD is satisfied if and only if our instance of MECPD is satisfied.
\end{proof}

\begin{remark}
Even though we utilize the above Lemma primarily to analyze the computational complexity of the problems, we will also utilize the construction provided for path planners for Min-Exp-Cost-Path-NT in Section \ref{sec:numerical_results}.
\end{remark}

We next show that MECPNTD is \RC{NP-complete (NP-hard and in NP)}, which together with Lemma \ref{lemma:mecpntd_red_mecpd}, implies that MECPD is NP-hard.

\begin{theorem}\label{theorem:mecpntd_np_hard}
Min-Exp-Cost-Path-NT Decision problem is \RC{NP-complete}.
\end{theorem}

\begin{proof}
\RC{Clearly MECPNTD is in NP, since given a path we can compute its associated expected cost in polynomial time.}
We \RC{next} show that MECPNTD is NP-hard using a reduction from a rooted version of the NP-hard Hamiltonian path problem \cite{garey2002computers}.
\AM{Does a rooted Hamiltonian path problem exist in literature?
Maybe we could instead add a node $v_s$ with zero edge weights to all other nodes.}
Consider an instance of the Hamiltonian path problem $\mathcal{G}=(\mathcal{V}, \mathcal{E})$,
where the objective is to determine if there exists a path originating from $v_s$ that visits each vertex only once.
We create an instance of MECPNTD by setting the probability of success to a non-zero constant for all nodes, i.e., $p_v = p>0$, $\forall v \in \mathcal{V}$.
We create a complete graph and set edge weights as 
$
l_{e} = \left\{\begin{array}{ll} 
1, & \text{if } e \in \mathcal{E} \\
2, & \text{else}  
\end{array}\right..
$

A Hamiltonian path $\mathcal{P}$ on $\mathcal{G}$, if it exists, would have an expected distance cost of
\begin{align*}
\sum_{e \in \mathcal{E}(\mathcal{P})} \left[\prod_{v \in \mathcal{V}(\mathcal{P}_{e})} (1-p_{v})\right] l_{e} 
& = \frac{1-p}{p}\left(1-(1-p)^{|\mathcal{V}|-1}\right).
\end{align*}
Any path on the complete graph that is not Hamiltonian on $\mathcal{G}$, would involve either more edges or an edge with a larger cost than $1$ and would thus have a cost strictly greater than that of $\mathcal{P}$.
Thus, by setting $B_{\text{MECPNT}} = \frac{1-p}{p}\left(1-(1-p)^{|\mathcal{V}|-1}\right)$, there exists a Hamiltonian path if and only if the specific MECPNTD instance created is satisfied.
Thus, the general MECPNTD problem is at least as hard as the Hamiltonian path problem.
Since the Hamiltonian path problem is NP-hard, this implies that MECPNTD is NP-hard.
\end{proof}

\begin{corollary}
Min-Exp-Cost-Path Decision problem is \RC{NP-complete}.
\end{corollary}

\begin{proof}
\RC{We can see that MECPD is in NP.
The proof of NP-hardness  follows directly} from Lemma \ref{lemma:mecpntd_red_mecpd}.
\RC{MECPD is thus NP-complete.}
\end{proof}

\subsection{Min-Exp-Cost-Simple-Path}\label{subsec:min_exp_cost_simple_path}
We now propose ways to tackle the prohibitive computational complexity (space complexity) of our MDP formulation of Section \ref{subsec:mdp_formulation}, which possesses a state space of size exponential in the number of nodes in the graph.
If we can restrict ourselves to paths that do not revisit nodes, known as \emph{simple paths} (i.e., cycle free paths), then the expected cost from a node could be expressed in terms of the expected cost from the neighboring node that the path visits next.\footnote{
Note that depending on how we impose a simple path, we may need to keep track of the visited nodes.
However, as we shall see, this keeping track of the history will not result in an exponential memory requirement, as was the case for the original MDP formulation.
We further note that it is also possible to impose simple paths without a need to keep track of the history of the visited nodes, as we shall see in Section \ref{subsec:idag}.}
We refer to this problem of minimizing the expected cost, while restricted to the space of simple paths, as the \emph{Min-Exp-Cost-Simple-Path} problem.
The Min-Exp-Cost-Simple-Path problem is also computationally hard as shown in the following Lemma.

\begin{lemma}\label{lemma:mecspd_np_hard}
The decision version of Min-Exp-Cost-Simple-Path is NP-hard.
\end{lemma}
\begin{proof}
This follows from Theorem \ref{theorem:mecpntd_np_hard} and Lemma \ref{lemma:mecpntd_red_mecpd}, since the optimal path considered in the construction of Theorem \ref{theorem:mecpntd_np_hard} was a simple path that visited all nodes.
\end{proof}

Note that the optimal path of Min-Exp-Cost-Path could involve revisiting nodes, implying that the optimal solution to Min-Exp-Cost-Simple-Path on $\mathcal{G}$ could be suboptimal.
For instance, consider the toy problem of Fig \ref{fig:counter_example}.
The optimal path starting from node $2$, in this case, is $\mathcal{P}^{*} = (2,1,2,3,4)$.

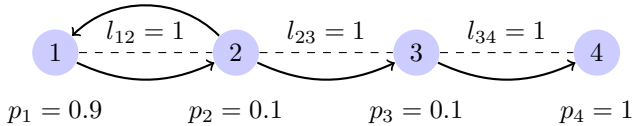
\begin{figure}[ht]
\centering

\begin{tikzpicture}
  [scale=.6,auto=left,main node/.style={circle,fill=blue!20}]
  \node[main node] (n1) at (1,1) {$1$};
  \node[main node] (n2) at (5,1) {$2$};
  \node[main node] (n3) at (9,1) {$3$};
  \node[main node] (n4) at (13,1) {$4$};
  \node (p1) [below=0.2cm of n1] {$p_1=0.9$};
  \node (p2) [below=0.2cm of n2] {$p_2=0.1$};
  \node (p3) [below=0.2cm of n3] {$p_3=0.1$};
  \node (p4) [below=0.2cm of n4] {$p_4=1$};
  
  \path (n1) edge[dashed] node[above] {$l_{12}=1$} (n2);   
  \path (n2) edge [dashed] node[above] {$l_{23}=1$} (n3);  
  \path (n3) edge [dashed] node[above] {$l_{34}=1$} (n4); 

  \path (n2) edge [->, thick, bend right =45] (n1);   
  \path (n1) edge [->, thick, bend right =25] (n2);  
  \path (n2) edge [->, thick, bend right =25] (n3); 
  \path (n3) edge [->, thick, bend right =25] (n4); 
\end{tikzpicture}

\caption{\small{A toy example with the optimal path from node $2$.
The optimal path involves revisiting node $2$.}}

\label{fig:counter_example}
\end{figure}

Consider Min-Exp-Cost-Simple-Path on the following complete graph.
This complete graph $\mathcal{G}_{\text{comp}}$ is formed from the original graph $\mathcal{G}=(\mathcal{V},\mathcal{E})$ by adding an edge between all pairs of vertices of the graph, excluding self-loops.
The cost of the edge $(u,v)$ is the cost of the shortest path between $u$ and $v$ on $\mathcal{G}$ which we denote by $l_{uv}^{\text{min}}$.
This can be computed by the all-pairs shortest path Floyd-Warshall algorithm in $O(|\mathcal{V}|^3)$ computations.
We next show in the following Lemma that the optimal solution of Min-Exp-Cost-Simple-Path on this complete graph can provide us with the optimal solution to Min-Exp-Cost-Path on the original graph.

\begin{lemma}\label{lemma:simple_path_complete_graph}
The solution to Min-Exp-Cost-Simple-Path on $\mathcal{G}_{\text{comp}}$ can be used to obtain the solution to Min-Exp-Cost-Path on $\mathcal{G}$.
\end{lemma}

\begin{proof}
See Appendix \ref{appendix:lemma_simple_path} for the proof.
\end{proof}

Lemma \ref{lemma:simple_path_complete_graph} is a powerful result that allows us to asymptotically solve the Min-Exp-Cost-Path problem, with $\epsilon$ sub-optimality, as we shall see in the next Section.

\section{Asymptotically $\epsilon$-suboptimal Path Planner}\label{sec:asmpt_near_opt_planner}
In this section, we propose a path planner, based on a game theoretic framework, that asymptotically gets arbitrarily close to the optimum solution of the Min-Exp-Cost-Path problem, i.e., it is an \emph{asymptotically $\epsilon$-suboptimal solver}.
This is important as it allows us to solve the NP-hard Min-Exp-Cost-Path problem, with near optimality, given enough time.
More specifically, we utilize log-linear learning to asymptotically obtain the global potential minimizer of an appropriately defined potential game.

We start with the space of simple paths, i.e., we are interested in the Min-Exp-Cost-Simple-Path problem on a given graph $\mathcal{G}$.
A node $v$ will then route to a single other node.
Moreover, the expected cost from a node can then be expressed in terms of the expected cost from the neighbor it routes through.
The state of the system can then be considered to be just the current node $v$, and the actions available at state $v$, $A_v = \{u\in \mathcal{V}:(v,u) \in \mathcal{E}\}$, is the set of neighbors of $v$.
The policy $\mu$ specifies which node to move to next, i.e., if the current node is $v$, then $\mu_{v}$ is the next node to go to.

We next discuss our game-theoretic setting.
So far, we viewed a node $v$ as a state and $A_v$ as the action space for state $v$.
In contrast, in this game-theoretic setting, we interpret node $v$ as a player and $A_v$ as the action set of player $v$.
Similarly, $\mu$ was viewed as a policy with $\mu_{v}$ specifying the action to take at state $v$.
Here, we reinterpret $\mu$ as the joint action profile of the players with $\mu_v$ being the action of player $v$.

We consider a game $\{\mathcal{V}', \{A_v\}, \{\mathcal{J}_{v}\}\}$, where the set of non-terminal nodes $\mathcal{V}'$ are the players of the game and $A_v$ is the action set of node/player $v$.
Moreover, $\mathcal{J}_{v}:A\rightarrow \mathbb{R}$ is the local cost function of player $v$, where $A = \prod_{v \in \mathcal{V}'} A_{v}$ is the space of joint actions.
Finally, $\mathcal{J}_{v}(\mu)$ is the cost of the action profile $\mu$ as experienced by player $v$.

We first describe the expected cost from a node $v$ in terms of the action profile $\mu$.
An action profile $\mu$ induces a directed graph on $\mathcal{G}$, which has the same set of nodes as $\mathcal{G}$ and directed edges from $v$ to $\mu_{v}$ for all $v\in \mathcal{V}'$.
We call this the successor graph, using terminology from \cite{garcia1993loop}, and denote it by $\mathcal{SG}(\mu)$.
As we shall show, our proposed strategy produces an action profile $\mu$ which induces a directed acyclic graph.
This is referred to as an \emph{acyclic successor graph} (ASG) \cite{garcia1993loop}.

\begin{figure}
\centering
\begin{tikzpicture}
  [scale=.6,auto=left,main node/.style={circle,fill=blue!20}]
  \node[main node] (n7) at (7,11) {$7$};
  \node[main node] (n6) at (1,10) {$6$};
  \node[main node] (n4) at (4,8)  {$4$};
  \node[main node] (n5) at (8,9)  {$5$};
  \node[main node] (n1) at (2,6) {$1$};
  \node[main node] (n3) at (9,6)  {$3$};
  \node[main node] (n2) at (6,5)  {$2$};
  
  \node (p1) [below=0.1cm of n1] {$p_{1}<1$};
  \node (p2) [below=0.1cm of n2] {$p_{2}<1$};
  \node (p7) [right=0.1cm of n7] {$p_{7}=1$};
  
  \path (n1) edge [->, semithick](n2);   
  \path (n2) edge [->, semithick] (n4);  
  \path (n4) edge [->, semithick] (n6); 
  \path (n6) edge [->, semithick] (n7); 
  \path (n5) edge [->, semithick] (n4); 
  \path (n3) edge [->, semithick] (n5); 
  

\end{tikzpicture}
\caption{\small{An example ASG induced by an action profile $\mu$.}}
\label{fig:ASG}
\end{figure}
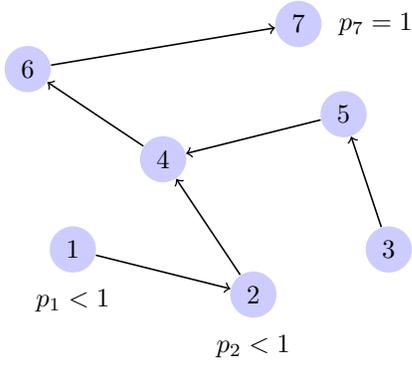

\RC{
Node $v$ is said to be \emph{downstream} of $u$ in $\mathcal{SG}(\mu)$ if $v$ lies on the directed path from $u$ to the corresponding sink.
Moreover, node $u$ is said to be \emph{upstream} of $v$ in this case, and we denote the set of upstream nodes of $v$ by $U_{v}(\mu_{-v})$,
where $\mu_{-v}$ denotes the action profile of all players except $v$.
Let $v \in U_v(\mu_{-v})$ by convention.
Note that $U_{v}(\mu_{-v})$ is only a function of $\mu_{-v}$ as it does not depend on the action of player $v$.}
\AM{Do we include the current node in the set of downstream nodes and upstream nodes?}

Let $\mathcal{P}(\mu, v)$ be the path from agent $v$ on this successor graph.
We use the shorthand $C_v(\mu) = C(\mathcal{P}(\mu, v),1)$, to denote the expected cost from node $v$ when following the path $\mathcal{P}(\mu, v)$.
Since $\mathcal{P}(\mu, v)$ is a path along $\mathcal{SG}(\mu)$, it can either end at some node or it can end in a cycle.
If it ends in a cycle or at a node that is not a terminal node, we define the expected cost $C_v(\mu)$ to be infinity.
If it does end at a terminal node, we obtain the following recursive relation from (\ref{eq:min_exp_cost_recursion}):
\begin{align}\label{eq:rec_exp_dist}
C_{v}(\mu) = (1-p_{v})\left(l_{v\mu_v} + C_{\mu_{v}}(\mu)\right),
\end{align}
where $C_{v_t}(\mu) = 0$ for all $v_t \in T$.

\RC{
Let $A_{\text{ASG}}$ denote the set of action profiles such that the expected cost $C_{v}(\mu) < \infty$ for all $v \in \mathcal{V}$.
This will only happen if the path $\mathcal{P}(\mu, v)$ ends at a terminal node for all $v$.
This corresponds to $\mathcal{SG}(\mu)$ being an ASG with terminal nodes as sinks.
Specifically, $\mathcal{SG}(\mu)$ would be a forest with the root or sink of each tree being a terminal node.
An ASG is shown in Fig. \ref{fig:ASG} for the toy example from Fig. \ref{fig:problem_setup}.
}

\RC{
$\mu \in A_{\text{ASG}}$ implies that the action of player $v$ satisfies $\mu_{v} \in A_{v}^{c}(\mu_{-v})$, where
$A_{v}^{c}(\mu_{-v}) = \{u \in \mathcal{V}: (v,u) \in \mathcal{E}, u \notin U_{v}(\mu_{-v}), C_u(\mu) < \infty\}$ is the set of actions that result in a finite expected cost from $v$.
Note that $A_{v}^{c}(\mu_{-v})$ is a function of only $\mu_{-v}$.
This is because $u \notin U_{v}(\mu_{-v})$ implies $v \notin \mathcal{P}(\mu, u)$ which in turn implies that $C_{u}(\mu)$ is a function of only $\mu_{-v}$.
}

We next define the local cost function of player $v$ to be 
\begin{align}\label{eq:local_cost_func}
\mathcal{J}_{v}(\mu) = \sum_{u \in U_{v}(\mu)}\alpha_{u}C_{u}(\mu),
\end{align}
where $U_v(\mu)$ is the set of upstream nodes of $v$, and $\alpha_u>0 $ are constants such that $\alpha_{v_s} = 1$ and $\alpha_v = \epsilon^{'}$, for all $v\neq v_s$, where $\epsilon^{'} > 0$ is a small constant.

We next show that these local cost functions induce a potential game \RC{over the action space $A_{\text{ASG}}$}.
In order to do so, we first define a potential game \RC{over $A_{\text{ASG}}$}.\footnote{\RC{This differs from the usual definition of a potential game in that the joint action profiles are restricted to lie in $A_{\text{ASG}}$.}}

\begin{definition}[Potential Game \cite{monderer1996potential}]
$\{\mathcal{V}', \{A_v\}, \{\mathcal{J}_v\}\}$ is an exact potential game \RC{over $A_{\text{ASG}}$} if there exists a function \RC{$\phi: A_{\text{ASG}} \rightarrow \mathbb{R}$} such that 
\begin{align*}
\mathcal{J}_v(\mu_v^{'}, \mu_{-v}) - \mathcal{J}_{v}(\mu_v, \mu_{-v}) = \phi(\mu_v^{'}, \mu_{-v}) - \phi(\mu_v, \mu_{-v}),
\end{align*}
for all \RC{$\mu_{v}^{'} \in A_{v}^{c}(\mu_{-v}), \mu=(\mu_v, \mu_{-v})\in A_{\text{ASG}}$}, and $v \in \mathcal{V}'$, where $\mu_{-v}$ denotes the action profile of all players except $v$.
\end{definition}

The function $\phi$ is called the potential function.
In the following Lemma, we show that using local cost functions as described in (\ref{eq:local_cost_func}), results in an exact potential game.
\begin{lemma}\label{lemma:potential_game}
The game $\{\mathcal{V}', \{A_v\}, \{\mathcal{J}_v\}\}$, with local cost functions as defined in (\ref{eq:local_cost_func}), is an exact potential game \RC{over $A_{\text{ASG}}$} with potential function 
\begin{align}\label{eq:pot_func}
\phi(\mu) = \sum_{v \in \mathcal{V}'}\alpha_{v}C_{v}(\mu) = C_{v_s}(\mu) + \epsilon^{'} \sum_{v\neq v_s} C_{v}(\mu).
\end{align}
\end{lemma}

\begin{proof}
Consider a node $v$ and $\mu = (\mu_v, \mu_{-v})$ and $\mu_{v}^{'}$ such that 
$C_v(\mu_v^{'}, \mu_{-v}) < C_v(\mu_v, \mu_{-v})$.
From (\ref{eq:rec_exp_dist}), we have that 
$
C_u(\mu_v^{'}, \mu_{-v}) < C_u(\mu_v, \mu_{-v}), \;\; \forall u \in U_{v}(\mu),
$
where $U_{v}(\mu)$ is the set of upstream nodes from $v$.
Furthermore,
$
C_u(\mu_v^{'}, \mu_{-v}) = C_u(\mu_v, \mu_{-v}), \;\; \forall u \notin U_{v}(\mu).
$
Thus, we have
\begin{align*}
\phi(\mu_v^{'}, \mu_{-v}) - \phi(\mu)  & = \sum_{u \in \mathcal{V}'}\alpha_u \left[C_{u}(\mu_v^{'}, \mu_{-v}) - C_{u}(\mu) \right] \\
& = \sum_{u \in U_v(\mu)}\alpha_u \left[C_{u}(\mu_v^{'}, \mu_{-v}) - C_{u}(\mu)\right]\\
& = \mathcal{J}_{v}(\mu_v^{'}, \mu_{-v}) - \mathcal{J}_{v}(\mu),
\end{align*}
for all \RC{$\mu_{v}^{'} \in A^{c}_v(\mu_{-v})$, $\mu \in A_{\text{ASG}}$}, and $v \in \mathcal{V}'$.
\AM{Do the infinities involved cause a problem? 
For instance, if $\mathcal{J}_i(\mu)$ is finite for some $i$, but $\phi(\mu)=\infty$, then is it still an ordinal potential game?}
\end{proof}

Minimizing $\phi(\mu)$ gives us a solution that can be arbitrarily close to that of Min-Exp-Cost-Simple-Path since we can select the value of $\epsilon^{'}$ appropriately. 
Let $\mu^{*} = \argmin_{\mu} \phi(\mu)$ and $\mu^{\text{OPT}} = \argmin_{\mu} C_{v_s}(\mu)$.
Then,
$
C_{v_{s}}(\mu^{*}) + \epsilon^{'} \sum_{u \neq v_s} C_{u}(\mu^{*}) \leq C_{v_s}(\mu^{\text{OPT}}) + \epsilon^{'} \sum_{u \neq v_s} C_{u}(\mu^{\text{OPT}}).
$
Rearranging gives us
\begin{align*}
C_{v_s}(\mu^{*}) & \leq C_{v_s}(\mu^{\text{OPT}}) + \epsilon^{'} \left[\sum_{u \neq v_s} C_{u}(\mu^{\text{OPT}}) - \sum_{u \neq v_s} C_{u}(\mu^{*})\right]\\
& \leq C_{v_s}(\mu^{\text{OPT}}) + \epsilon^{'} |\mathcal{V}'|D,
\end{align*}
where $D$ is the diameter of the graph.
Thus minimizing $\phi(\mu)$ gives us an $\epsilon$-suboptimal solution to the Min-Exp-Cost-Simple-Path problem, where $\epsilon=\epsilon^{'} |\mathcal{V}'|D$.

We next show how to asymptotically obtain the global minimizer of $\phi(\mu)$ by utilizing a learning process known as log-linear learning \cite{marden2012revisiting}.

\subsubsection{Log-linear Learning}\label{subsubsec:log_linear_learning}
\RC{
Let $\mu_{v} = a_{\emptyset}$ correspond to node $v$ not pointing to any successor node.
We refer to this as a null action.}
Then, the log-linear process utilized in our setting is as follows:
\begin{enumerate}
\item The action profile $\mu(0)$ is initialized with a null action, i.e., \RC{$\mu_v(0) = a_{\emptyset}$ for all $v$}.
The local cost function is thus $\mathcal{J}_{v}(\mu(0)) = \infty$, for all $v \in \mathcal{V}'$.
\item At every iteration $k+1$, a node $v$ is randomly selected 
\RC{from $\mathcal{V}'$ uniformly}.
\RC{If $A_{v}^{c}(\mu_{-v}(k))$ is empty, we set $\mu_v(k+1) = a_{\emptyset}$.
Else, node $v$ selects action $\mu_v(k+1)=\mu_v \in A_{v}^{c}(\mu_{-v}(k))$ with the following probability: 
\begin{align*}
\mathrm{Pr}(\mu_v)= \frac{e^{-\frac{1}{\tau}\left(\mathcal{J}_v(\mu_v, \mu_{-v}(k))\right)}}{\sum_{\mu_{v}^{'} \in A_v^{c}(\mu_{-v}(k))}e^{-\frac{1}{\tau}\left(\mathcal{J}_{v}(\mu_v^{'}, \mu_{-v}(k))\right)}},
\end{align*}
}
where $\tau$ is a tunable parameter known as the temperature.
The remaining nodes repeat their action, i.e., $\mu_{u}(k+1) = \mu_{u}(k)$ for $u\neq v$.
\end{enumerate}

We next show that log-linear learning asymptotically obtains an $\epsilon$-suboptimal solution to the Min-Exp-Cost-Path problem.
We first show, in the following Lemma, that it asymptotically provides an $\epsilon$-suboptimal solution to the Min-Exp-Cost-Simple-Path problem.
\begin{theorem}\label{theorem:log_lin_asympt}
As $\tau \rightarrow 0$, log-linear learning on a potential game with a local cost function defined in (\ref{eq:local_cost_func}), asymptotically provides an $\epsilon$-suboptimal solution to the Min-Exp-Cost-Simple-Path problem.
\end{theorem}
\begin{proof}
See Appendix \ref{appendix:proof_log_lin_asympt} for the proof.
\end{proof}

\begin{lemma}
As $\tau \rightarrow 0$, log-linear learning on a potential game with a local cost function defined in (\ref{eq:local_cost_func}) on the complete graph $\mathcal{G}_{\text{comp}}$, asymptotically provides an $\epsilon$-suboptimal solution to the Min-Exp-Cost-Path problem.
\end{lemma}
\begin{proof}
From Theorem \ref{theorem:log_lin_asympt}, we know that log-linear learning asymptotically provides an $\epsilon$-suboptimal solution to the Min-Exp-Cost-Simple-Path problem on the complete graph $\mathcal{G}_{\text{comp}}$.
Using Lemma \ref{lemma:simple_path_complete_graph}, we then utilize this solution to obtain an $\epsilon$-suboptimal solution to the Min-Exp-Cost-Path problem on $\mathcal{G}$.
\end{proof}

\RC{
\begin{remark}
We implement the log-linear learning algorithm by keeping track of the expected cost $C_{v}(\mu(k))$ in memory, for all nodes $v \in \mathcal{V}'$.
In each iteration, we compute the set of upstream nodes of the selected node $v$ in order to compute the set $A_v^{c}(\mu(-k))$.
From (\ref{eq:rec_exp_dist}), we can see that the expected cost of each node upstream of $v$ can be expressed as a linear function of $C_{v}(\mu)$.
Then we can compute an expression for $\mathcal{J}_{v}(\mu) = \sum_{u \in \mathcal{U}_{v}(\mu_{-v})}C_{u}(\mu)$ as a linear function of the expected cost $C_{v}(\mu)$ with a computational cost of $O(|\mathcal{V}'|)$.
We can then compute $\mathcal{J}_{v}(\mu_v, \mu_{-v})$ for all $\mu_v \in A_v^{c}(\mu(-k))$ using this pre-computed expression for $\mathcal{J}_{v}()$.
Finally, once $\mu_{v}(k+1)$ is selected, we update the expected cost of $v$ and all its upstream nodes using (\ref{eq:rec_exp_dist}).
Thus, the overall computation cost of each iteration is $O(|\mathcal{V}'|)$.
\end{remark}
}

\AM{
\begin{remark}
Log-linear learning is closely related to the simulated annealing process \cite{kirkpatrick1983optimization}.
More precisely, from a centralized perspective, a variant of log-linear learning termed binary log-linear learning \cite{marden2012revisiting}, can be viewed as a type of simulated annealing.
\end{remark}

\begin{remark}
A natural question that arises is why we cannot employ some version of simulated annealing directly on our objective of $C_{v_s}(\mu)$.
The reason is that by setting a small weight $\epsilon^{'}$ for the expected cost $C_{u}(\mu)$ for all $u\neq v_s$, a self-organizing tendency is promoted and an ASG is naturally formed.
Moreover, since the expected cost from a node is directly related to the expected cost from the neighboring node it routes to, improving the expected cost from its neighboring nodes is a reasonable action.
If on the other hand, we had used $C_{v_s}(\mu)$ as the objective alone, large regions of the action space for $\mu$ would be infeasible and navigating through the action space would be more difficult.
As a result of this, the simulated annealing procedure would not be effective.
\AM{Reference for simulated annealing not performing well when there are large gaps in performance}
\end{remark}
}

\AM{Give intuition on how it makes sense that the utility of an agent is the sum of the weighted expected costs of all nodes downstream from it.
This is because its choice affects all nodes downstream of it.
For instance, if we were concerned with minimizing $C_{v_s}(\mu)$, then a node will be more risk averse if it has $v_s$ on its downstream path.
Otherwise, it is free to take larger risks, i.e., actions that can result in larger cost for it.}

\section{Fast Non-myopic Path Planners}\label{sec:non_myopic_planners}
In the previous section, we proposed an approach that finds an $\epsilon$-suboptimal solution to the Min-Exp-Cost-Path problem asymptotically.
However, for certain applications, finding a suboptimal but fast solution may be more important.
This motivates us to propose two suboptimal path planners that are non-myopic and very fast.
We use the term non-myopic here to contrast with the myopic approaches of choosing your next step based on your immediate or short-term reward (e.g., local greedy search).
We shall see an example of such a myopic heuristic in Section \ref{sec:numerical_results}.
 
In this part, we first propose a non-myopic path planner based on a game theoretic framework that finds a directionally local minimum of the potential function $\phi$ of (\ref{eq:pot_func}).
We next propose a path planner based on an SSP formulation that provides us with the optimal path among the set of paths satisfying a mild assumption.

We assume simple paths in this Section.
Lemma \ref{lemma:simple_path_complete_graph} can then be used to find a optimum non-simple path with minimal computation.
Alternatively, the simple path solution can also be directly utilized.

\subsection{Best Reply Process}\label{subsec:br_asg}
Consider the potential game $\{\mathcal{V}', \{A_v\}, \{\mathcal{J}_{v}\}\}$ of Section \ref{sec:asmpt_near_opt_planner} with local cost functions $\{\mathcal{J}_{v}\}$ as given in (\ref{eq:local_cost_func}).
We next show how to obtain a directionally local minimum of the potential function $\phi(\mu) =  C_{v_{s}}(\mu) + \epsilon^{'}\sum_{v \neq v_s}C_{v}(\mu)$.
In order to do so, we first review the definition of a Nash equilibrium.
\begin{definition}[Nash Equilibrium \cite{fudenberg1991game}]
An action profile $\mu^{\text{NE}}$ is said to be a pure Nash equilibrium if 
\begin{align*}
\mathcal{J}_{v}(\mu^{\text{NE}}) \leq \mathcal{J}_v(\mu_v, \mu_{-v}^{\text{NE}}), \;\; \forall \mu_v \in A_v, \forall v \in \mathcal{V}'
\end{align*}
where $\mu_{-v}$ denotes the action profile of all players except $v$.
\end{definition}

It can be seen that an action $\mu^{\text{NE}}$ is a Nash equilibrium of a potential game if and only if it is a directionally local minimum of $\phi$, i.e., 
$
\phi(\mu_v^{'}, \mu_{-v}^{\text{NE}}) \geq \phi(\mu^{\text{NE}}), \;\; \forall \mu_{v}^{'} \in A_v, \forall v \in \mathcal{V}'.
$
Since we have a potential game, a Nash equilibrium of the game is a directionally local minimum of $\phi(\mu)$.
We can find a Nash equilibrium of the game using a learning mechanism such as the best reply process \cite{marden2012revisiting}, which we next discuss.

\RC{
Let $\mu_{v} = a_{\emptyset}$ correspond to node $v$ not pointing to any successor node.
We refer to this as a null action.
}
The best reply process utilized in our setting is as follows:
\begin{enumerate}
\item The action profile $\mu(0)$ is initialized with a null action, i.e., \RC{$\mu_v(0) = a_{\emptyset}$ for all $v$}.
The local cost function is thus $\mathcal{J}_{v}(\mu(0)) = \infty$, for all $v \in \mathcal{V}'$.
\item At iteration $k+1$, a node $v$ is randomly selected from $\mathcal{V}'$ \RC{uniformly}.
\RC{If $A_{v}^{c}(\mu_{-v}(k))$ is empty, we set $\mu_v(k+1) = a_{\emptyset}$.
Else,}
the action of node $v$ is updated as 
\begin{align*}
\mu_{v}(k+1) & = \argmin_{\mu_v \in \RC{A_v^{c}(\mu_{-v}(k))}} \mathcal{J}_{v}(\mu_{v}, \mu_{-v}(k))\\
& = \argmin_{\mu_v \in \RC{A_v^{c}(\mu_{-v}(k))}} C_{v}(\mu_{v}, \mu_{-v}(k))\\
 & =\argmin_{u \in \RC{A_v^{c}(\mu_{-v}(k))}} \left\{ (1-p_v)\left[l_{vu} + C_{u}(\mu(k))\right] \right\},
\end{align*}
where the second and third equality follow from (\ref{eq:rec_exp_dist}).
The actions of the remaining nodes stay the same, i.e., $\mu_{u}(k+1) = \mu_{u}(k)$, $\forall u \neq v$.
\end{enumerate}

The best reply process in a potential game converges to a pure Nash equilibrium \cite{marden2012revisiting}, which is also a directionally local minimum of $\phi(\mu)=  C_{v_{s}}(\mu) + \epsilon^{'}\sum_{v \neq v_s}C_{v}(\mu)$.

Since a node is selected at random at each iteration in the best reply process, analyzing its convergence rate becomes challenging.
Instead, in the following Theorem, we analyze the convergence rate of the best reply process when the nodes for update are selected deterministically in a cyclic manner.
We show that it converges quickly to a directionally local minimum, and is thus an efficient path planner.
\begin{theorem}[Computational complexity]\label{theorem:BR_finite_iters}
Consider the best reply process where we \RC{select the next node for update in a round robin fashion}.
Then, this process converges after at most $|\mathcal{V}'|^2$ iterations.
\end{theorem}

\begin{proof}
See Appendix \ref{appendix:proof_br_conv} for the proof.
\end{proof}

\RC{
\begin{remark}
We implement the best reply process by keeping track of the expected cost $C_{v}(\mu(k))$  in memory, for all nodes $v \in \mathcal{V}'$.
In each iteration of the best reply process, we compute the set of upstream nodes of the selected node $v$ in order to compute the set $A_v^{c}(\mu(-k))$.
Moreover, we compute $l_{v\mu_{v}} + C_{\mu_v}(\mu(k))$ for all $\mu_v \in A_v^{c}(\mu(-k))$ to find the action $\mu_v$ that minimizes the expected cost from $v$.
Finally, once $\mu_{v}(k+1)$ is selected, we update the expected cost of $v$ as well as all the nodes upstream of it using (\ref{eq:rec_exp_dist}).
Then, the computation cost of each iteration is $O(|\mathcal{V}'|)$.
Thus, from Theorem \ref{theorem:BR_finite_iters}, the best reply process in a round robin setting has a computational complexity of $O(|\mathcal{V}'|^{3})$.
\end{remark}
}

\AM{
\begin{remark}
From a centralized perspective, the best reply process can be viewed as a type of greedy algorithm for minimizing the potential function $\phi(\mu)=  C_{v_{s}}(\mu) + \epsilon^{'}\sum_{v \neq v_s}C_{v}(\mu)$, which arrives at a locally optimum solution of $\phi(\mu)$.
\end{remark}
}

\subsection{Imposing a Directed Acyclic Graph}\label{subsec:idag}
We next propose an SSP-based path planner.
We enforce that a node cannot be revisited by imposing a directed acyclic graph (DAG), $\mathcal{G}_{\text{DAG}}$, on the original graph.
The state of the SSP formulation of Section \ref{subsec:mdp_formulation} is then just the current node $v \in \mathcal{V}'$.
The transition \RC{probability from state $v$ to state $u$ is then simply given as} $P_{vu}(a_v) = \left\{\begin{array}{lll} 1-p_v, & \text{if } u=f(a_v)\\ 
p_v, & \text{if } u=s_t\\
0, & \text{else}\end{array}\right.$,
where $f(a_v) = \left\{\begin{array}{ll} a_v, & \text{if } a_v \in \mathcal{V}'\\ 
s_{t}, & \text{if } a_v \in T\end{array}\right.$,
and the stage cost \RC{of action $u$ at state $v$} is given as $g(v,u)= (1-p_v)l_{vu}$.
We refer to running value iteration on this SSP as the IDAG (imposing a DAG) path planner.

Imposing a DAG, $\mathcal{G}_{\text{DAG}} = (\mathcal{V}, \mathcal{E}_{\text{DAG}})$, corresponds to modifying the action space of each state $v$ such that only a subset of the neighbors are available actions, i.e., $A_v = \{u:(v,u) \in \mathcal{E}_{\text{DAG}} \}$.
For instance, given a relative ordering of the nodes, a directed edge would be allowed from node $u$ to $v$, only if $v\geq u$ with respect to some ordering.
As a concrete example, consider the case where a directed edge from node $u$ to $v$ exists only if $v$ is farther away from the starting node $v_s$ on the graph than node $u$ is, i.e., $l^{\text{min}}_{v_sv}>l^{\text{min}}_{v_su}$, where $l^{\text{min}}_{v_sv}$ is the cost of the shortest path from $v_s$ to $v$ on the original graph $\mathcal{G}$.
More specifically, the imposed DAG has the same set of nodes $\mathcal{V}$ as the original graph, and the set of edges is given by $\mathcal{E}_{\text{DAG}} = \{(u,v) \in \mathcal{E}: l^{\text{min}}_{v_sv}>l^{\text{min}}_{v_su}\}$, where $(u,v)$ represents a directed edge from $u$ to $v$.
For example, consider an $n\times n$ grid graph, where neighboring nodes are limited to $\{\text{left, right, top, down}\}$ nodes.
In the resulting DAG, only outward flowing edges from the start node are allowed, i.e., edges that take you further away from the start node.
For instance, consider the start node $v_s$ as the center and for each quadrant, form outward moving edges, as shown in Fig. \ref{fig:SSP_DAG}.
In the first quadrant only right and top edges are allowed, in the second quadrant only left and top edges and so on.
Fig. \ref{fig:SSP_DAG} shows an illustration of this, where several feasible paths from  $v_s$ to a terminal node are shown.

\begin{figure}
\centering
\includegraphics[scale=0.25]{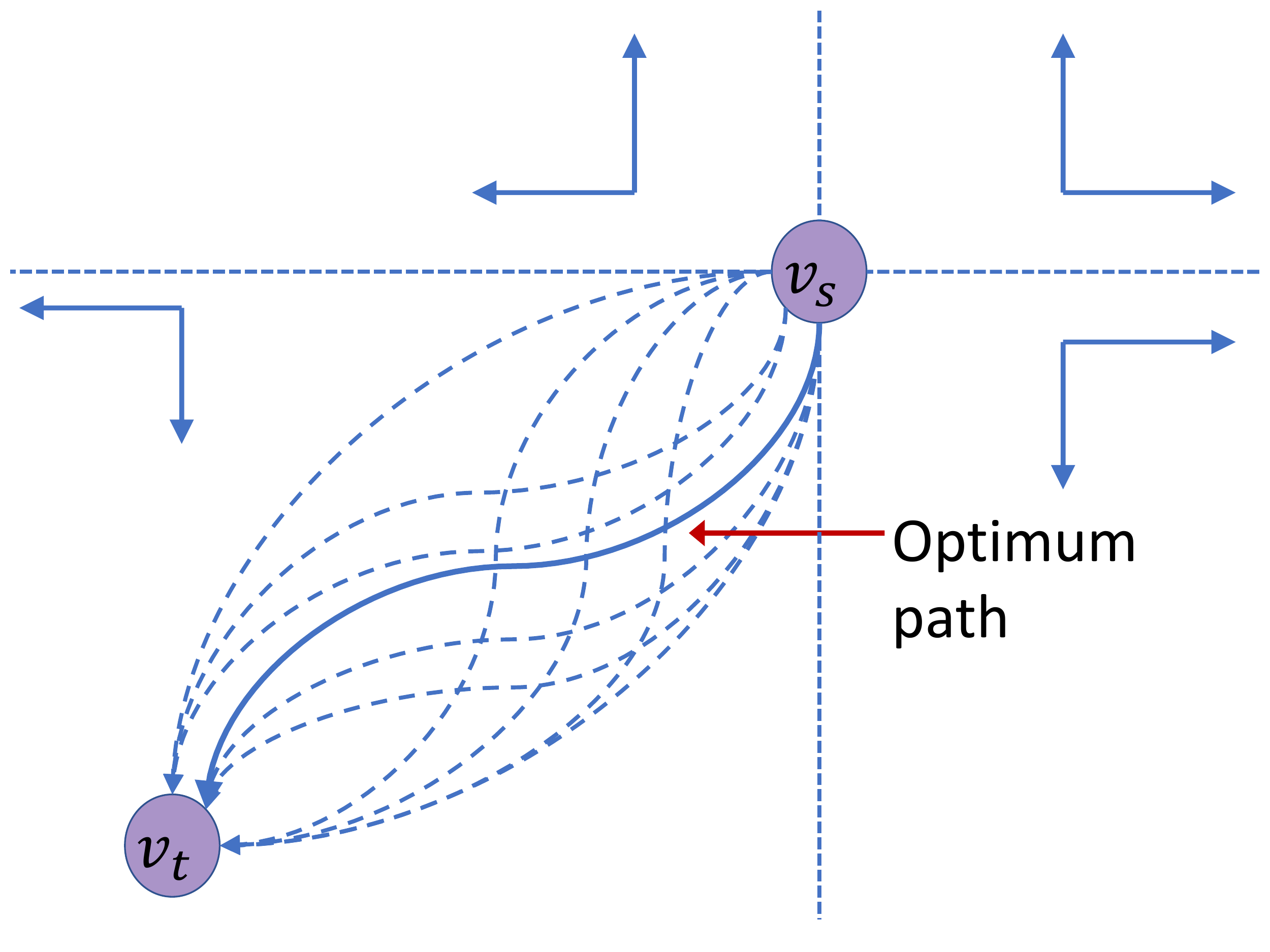}
\caption{\small{A DAG is imposed which allows only ``outward" motion.
The solution produced by SSP would be the best among all such paths from the start node $v_s$ to the terminal node $v_t$.}}
\label{fig:SSP_DAG}
\end{figure}

Imposing this DAG is equivalent to placing the following requirement that a feasible path must satisfy: \emph{Each successive node on the path must be further away from the starting node $v_s$}, i.e., for a path $\mathcal{P}=(v_1=v_s, v_2, \cdots, v_m)$, the condition $l^{\text{min}}_{v_sv_i} > l^{\text{min}}_{v_sv_{i-1}}$ should be satisfied.
In the case of a grid graph with a single terminal node, this implies that a path must always move towards the terminal node, which is a reasonable requirement to impose.
We next show that we can obtain the optimal solution among all paths satisfying this requirement using value iteration.

The optimal solution with minimum expected cost on the imposed DAG $\mathcal{G}_{\text{DAG}}$ can be found by running value iteration:
\begin{align*}
J_{v}(k+1)& = \min_{u \in A_v} \left\{(1-p_v)l_{vu} + (1-p_v)J_{f(u)}(k)\right\},
\end{align*}
with the policy at iteration $k+1$ given by
\begin{align*}
\mu_{v}(k+1)& = \argmin_{u \in A_v} \left\{(1-p_v)l_{vu} + (1-p_v)J_{f(u)}(k)\right\},
\end{align*}
for all $v \in \mathcal{V}'$, where $J_{s_t}(k) = 0$, for all $k$.

\begin{table*}[t]
\centering
\resizebox{0.8\linewidth}{!}{
\begin{tabular}{c|c|c|c|c|c|c|c|c}
Size of grid ($n$) & $5$ & $10$ & $15$ & $20$ & $25$ & $30$ & $40$ & $50$\\ \hline
RTDP (MDP formulation) & $3.510$ & - & - & - & - & - & - & -\\ \hline 
Simulated Annealing & $3.510$ & $7.515$ & $84.138$ & $159.88$ & $150.59$ & $25.997$  & $426.71$ & $406.71$\\ \hline 
Best Reply & $3.510$ & $6.722$ & $8.588$ & $10.437$ & $10.569$ & $11.032$ & $12.930$ & $12.530$\\ \hline 
Log-linear & $3.510$ & $6.722$ & $8.588$ & $10.437$ & $10.569$ & $11.032$ & $13.450$ & $12.186$\\ 
\end{tabular}}
\small \caption{\small Expected traveled distance for the path produced by the various approaches for different grid sizes ($n$) with a single connected point (\RC{$n_t=1$}).
We can see that RTDP is unable to produce a viable path for $n\geq 10$ and that simulated annealing produces paths with poor performance for increasing $n$.}
\label{table:exp_dist_rtdp_sim_ann}
\end{table*}

The following lemma shows that we can find this optimal solution efficiently.
\begin{lemma}[Computational complexity]\label{lemma:idag_comp_complexity}
When starting from $J_{v}(0) = \infty$, for all $v \in \mathcal{V}'$, the value iteration method will yield the optimal solution after at most $|\mathcal{V}'|$ iterations.
\end{lemma}
\begin{proof}
This follows from the convergence analysis of value iteration on an SSP with a DAG structure \cite{bertsekas1995dynamic}.
\end{proof}

\RC{
\begin{remark}
Each stage of the value iteration process has a computation cost of $O(|\mathcal{E}_{\text{DAG}}|)$ since for each node we have as many computations as there are outgoing edges.
Thus, from Lemma \ref{lemma:idag_comp_complexity}, we can see that the computational cost of value iteration is $O(|\mathcal{V}'||\mathcal{E}_{\text{DAG}}|)$.
\end{remark}
}

\begin{remark}
\RC{Log-linear learning, best reply, and IDAG, each have their own pros and cons.
For instance, log-linear learning has strong asymptotic optimality guarantees.
In contrast, best reply converges quickly to a directionally-local minimum but does not possess similar optimality guarantees.
Numerically, for the applications considered in Section \ref{sec:numerical_results}, the best reply solver performs better than the IDAG solver.
However, the IDAG approach is considerably fast and provides a natural understanding of the solution it produces, being particularly suitable for spatial path planning problems.
For instance, as shown in Fig. \ref{fig:SSP_DAG}, the solution of IDAG for the imposed DAG is the best solution among all paths that move outward from the start node.
More generally, it is the optimal solution among all the paths allowed by the imposed DAG.}
\end{remark}

\section{Numerical Results}\label{sec:numerical_results}
In this section, we show the performance of our approaches for Min-Exp-Cost-Path, via numerical analysis of two applications.
In our first application, a rover is exploring mars, to which we refer as the SamplingRover problem.
In our second application, we then consider a realistic scenario of a robot planning a path in order to find a connected spot to a remote station.
We see that in both scenarios our solvers perform well and outperform the naive and greedy heuristic approaches.

\subsection{Sampling Rover}
The scenario considered here is loosely inspired by the RockSample problem introduced in \cite{smith2004heuristic}.
A rover on a science exploration mission is exploring an area looking for an object of interest for scientific studies.
For instance, consider a rover exploring Mars with the objective of obtaining a sample of water.
Based on prior information, which could for instance be from orbital flyovers over the area of interest or from the estimation by experts, the rover has an a priori probability of finding the object at any location.

An instance of the SamplingRover[$n$,\RC{$n_t$}] consists of an $n \times n$ grid with \RC{$n_t$} locations of guaranteed success, i.e., \RC{$n_t$} nodes such that $p_{v}=1$.
The probability of success at each node is generated independently and uniformly from $[0,0.1]$.
At any node, the actions allowed by the rover are $\{\text{left, right, top, down}\}$.
The starting position of the rover is taken to be at the center of the grid, $v_s = \left(\lfloor \frac{n}{2} \rfloor, \lfloor \frac{n}{2} \rfloor \right)$.
When the number of points of guaranteed success (\RC{$n_t$}) is $1$, we take the location of the node with $p_v=1$ at $(0,0)$.

\begin{figure}
    \centering
    \includegraphics[width=0.7\linewidth]{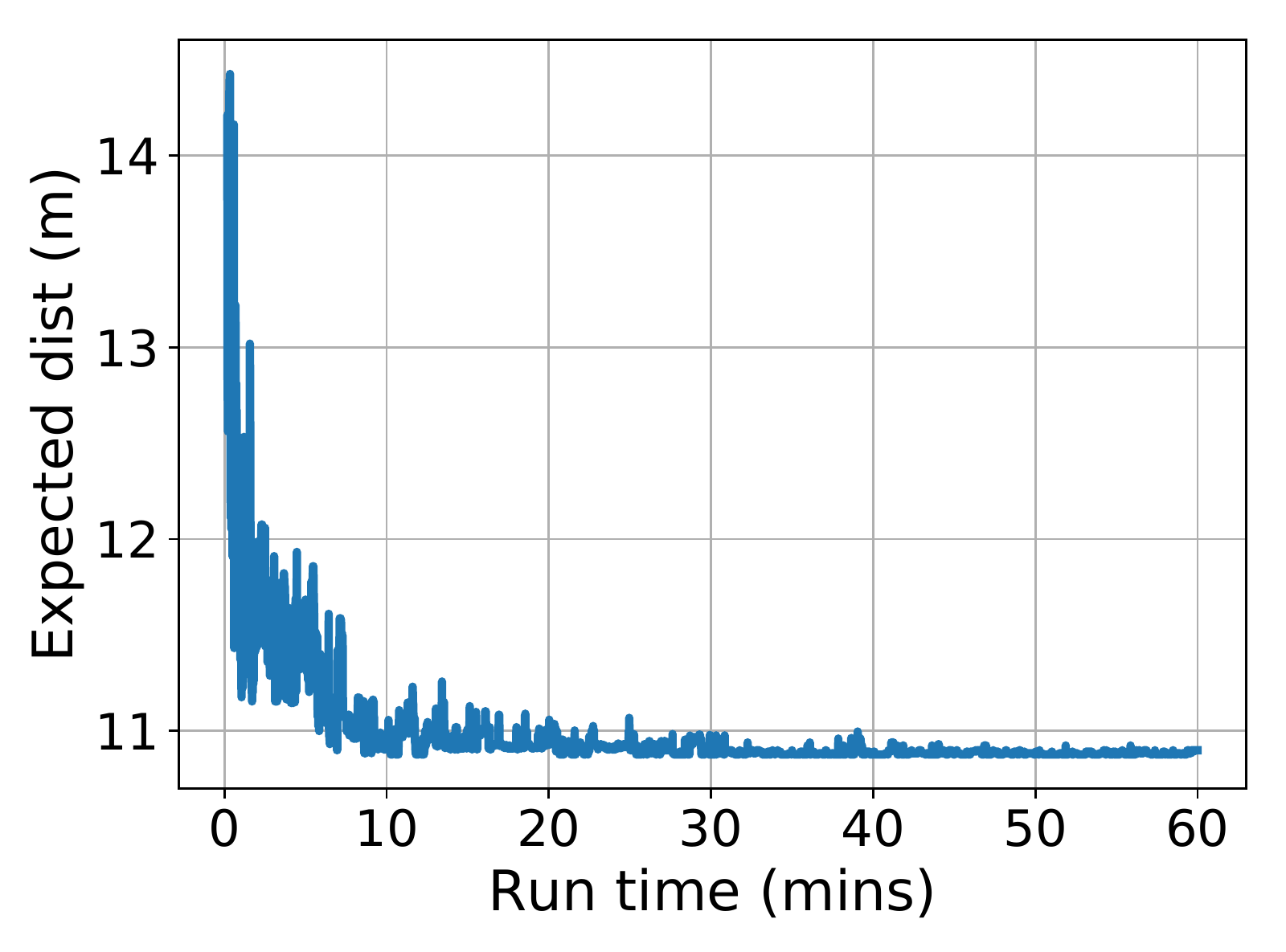}
 \caption{\small Evolution of the expected traveled distance with time for the log-linear learning approach with $n=25$ and \RC{$n_t=1$}.}
    \label{fig:exp_dist_time_log_lin}
\end{figure}

\begin{figure}
    \centering
    \includegraphics[width=0.75\linewidth]{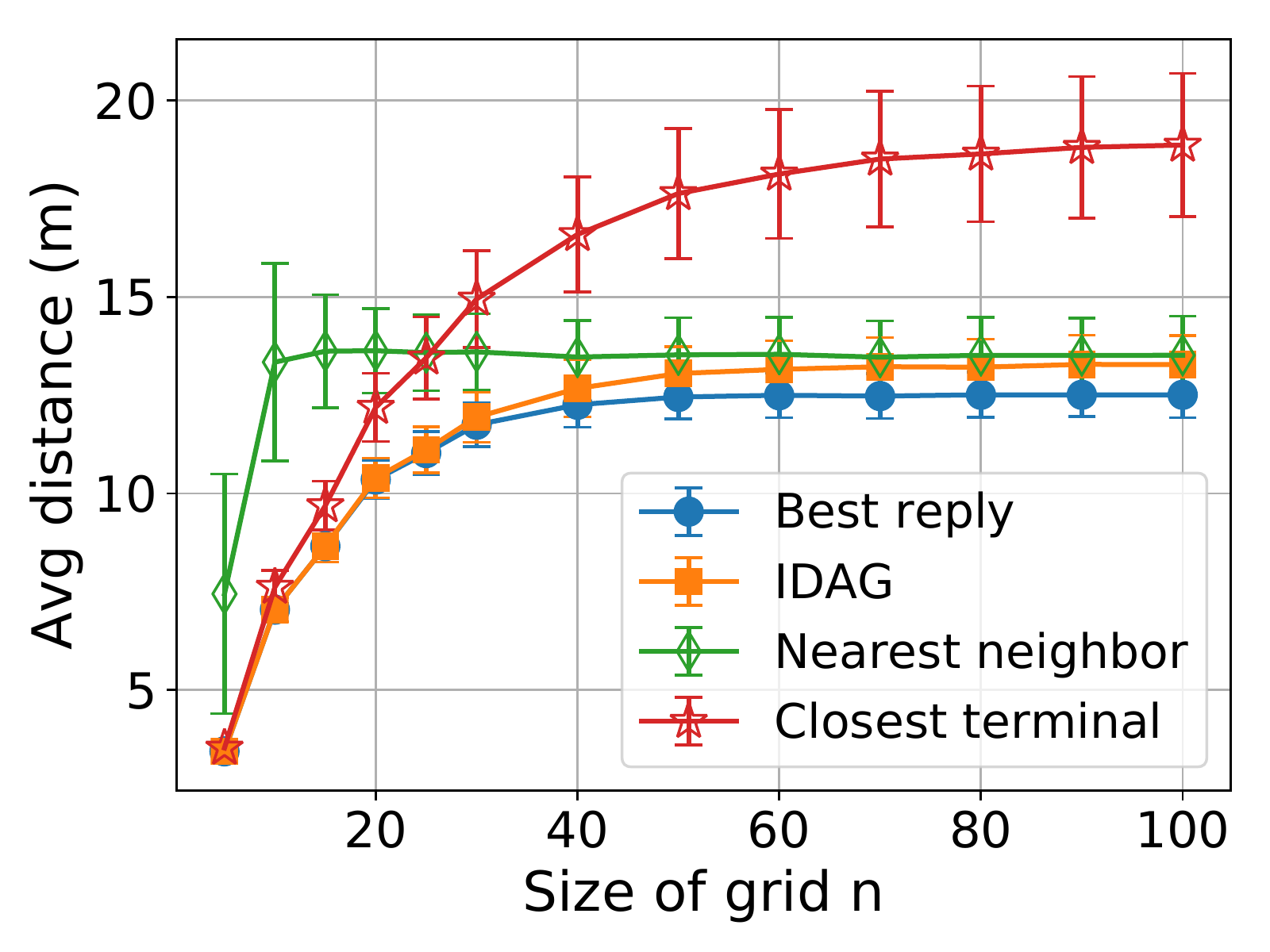}
 \caption{\small The expected traveled distance by the various approaches for different grid sizes ($n$) with a single connected point (\RC{$n_t=1$}).
The results are averaged over $1000$ different probability of success maps.
The corresponding standard deviation is also shown in the form of error bars.
We can see that the best reply and IDAG approaches outperform the greedy and closest terminal heuristics.}
    \label{fig:exp_dist_all_approaches}
\end{figure}

\begin{figure*}[ht]
    \centering
    \begin{subfigure}[b]{0.425\linewidth}
        \centering
        \includegraphics[width=\linewidth]{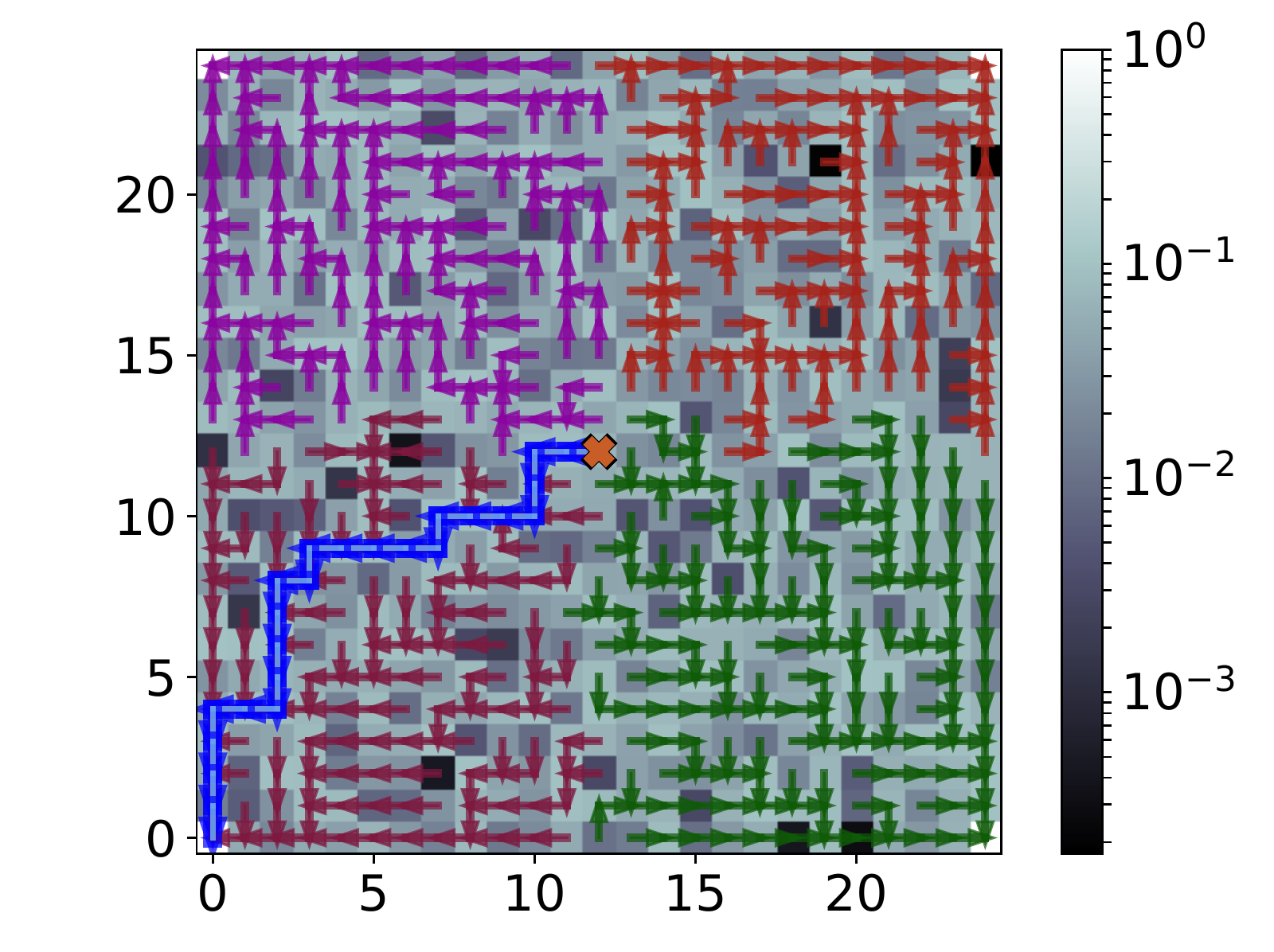}
    \end{subfigure}
    \hspace{0.2in}
    \begin{subfigure}[b]{0.425\linewidth}
        \centering
        \includegraphics[width=\linewidth]{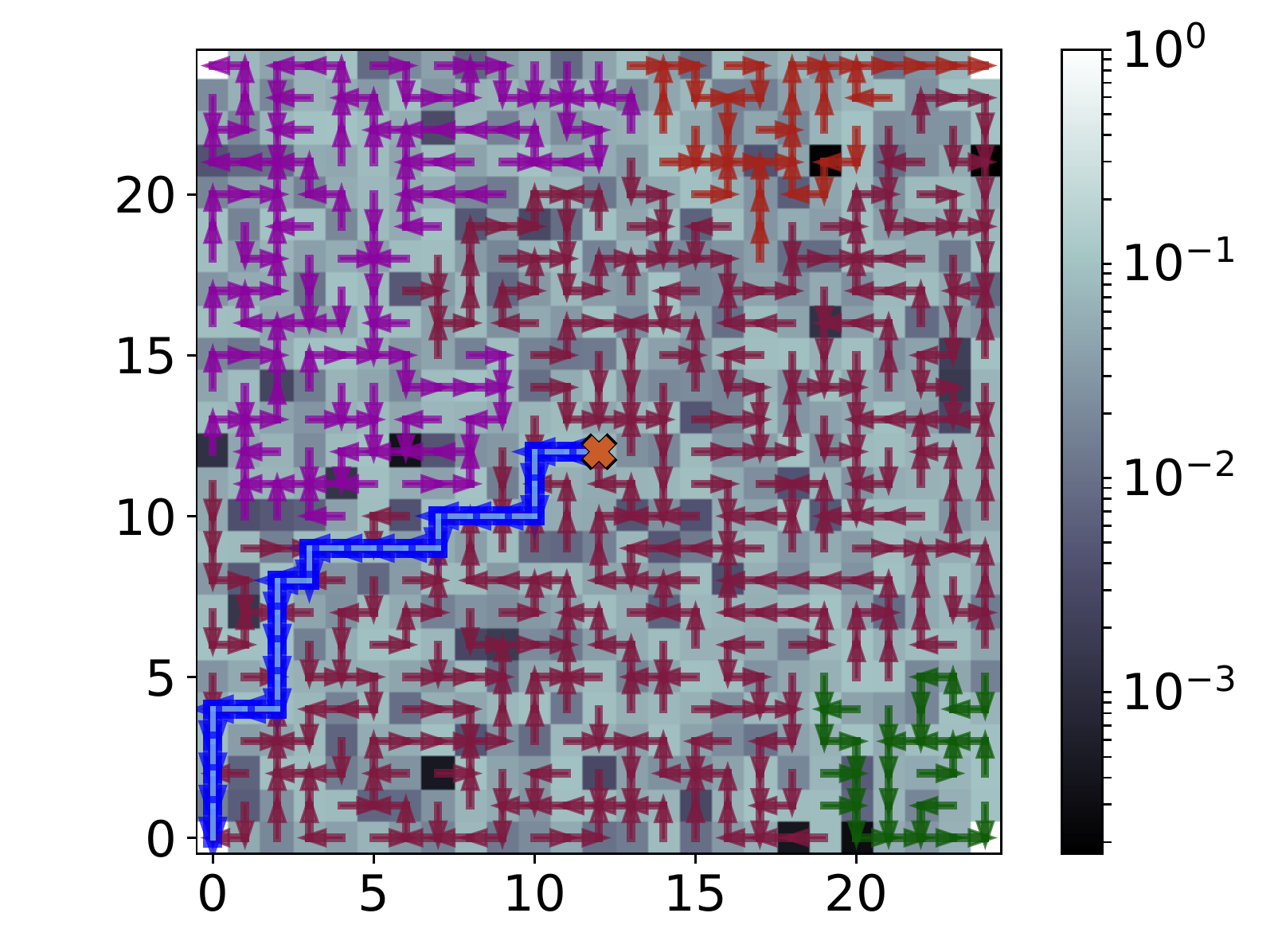}
    \end{subfigure}
\caption{\small Acyclic successor graph (ASG) of (left) best reply process and (right) log-linear learning process for $n=25$, when there are four nodes with $p_v=1$.
\RC{The trees corresponding to each of the four terminal nodes are marked in purple, red, green and brown.}
The path traveled from the starting node is also plotted (in blue).
The starting position at $(12,12)$ is marked by the orange ``x''.
The background color plot specifies the probability of success of each node.
Readers are referred to the color pdf for better visibility.}
    \label{fig:k_4_path_25_asg}
    \vspace{-0.1in}
\end{figure*}

We found that log-linear learning on a complete graph produces similar results as log-linear learning on the original grid graph, but over longer run-times.
Thus, unless explicitly mentioned otherwise, when we refer to the best reply or the log-linear learning approach, it is with respect to finding a simple path on the original grid graph.
We set weight $\epsilon^{'} = 10^{-6}$ in $\phi(\mu)=  C_{v_{s}}(\mu) + \epsilon^{'}\sum_{v \in \mathcal{V}', v \neq v_s}C_{v}(\mu)$.
We use a decaying temperature for log-linear learning.
Through experimentation, we found that a decaying temperature of $\tau \propto k^{-0.75}$ (where $k$ is the iteration number) performs well.

We first compare our approach with alternate approaches for solving the Min-Exp-Cost-Path problem.
We consider one instance of a probability of success map.
We then implement Real Time Dynamic Programming (RTDP) \cite{barto1995learning}, which is a heuristic search method that tries to obtain a good solution quickly for the MDP formulation of Section \ref{subsec:mdp_formulation}.
Furthermore, we also implemented Simulated Annealing as implemented in \cite{kirkpatrick1983optimization} for the traveling salesman problem, where we modify the cost of a state to be the expected cost from the starting node.
Moreover, the starting position of the rover is fixed as the start of the simulated annealing path.
Table \ref{table:exp_dist_rtdp_sim_ann} shows the performance of RTDP, simulated annealing and our (asymptotically  $\epsilon$-suboptimal) log-linear learning  and  (non-myopic fast) best reply  approaches for various grid sizes ($n$) when \RC{$n_t=1$}, where for each approach we impose a computational time limit of an hour.
We see that RTDP is unable to produce viable solutions for $n\geq 10$ due to the state explosion problem of the MDP formulation, as discussed in Section \ref{subsec:mdp_formulation}.
Moreover, the performance of simulated annealing worsens significantly with increasing values of $n$.
On the other hand, the best reply and log-linear learning approach produce solutions with good performance that outperform simulated annealing considerably (e.g., simulated annealing has $15$ times more expected traveled distance than the best reply approach for $n=20$).

We next show the asymptotically $\epsilon$-suboptimal behavior of the log-linear learning approach of Section \ref{sec:asmpt_near_opt_planner}.
Fig. \ref{fig:exp_dist_time_log_lin} shows the evolution of the expected distance with time for the solution produced by log-linear learning for an instance of a probability of success map with $n=25$ and \RC{$n_t=1$}.
\RC{In comparison, the best reply and IDAG approaches converged in $1.75$ s and $0.25$ s respectively.}
\begin{remark}
We note that based on several numerical results, we have observed that the best reply and IDAG approaches produce results very close to those produced by log-linear learning.
They thus act as fast efficient solvers.
On the other hand, the log-linear learning approach provides a guarantee of optimality (within $\epsilon$) asymptotically.
Thus all $3$ approaches are useful depending on the application requirements.
\end{remark}

We next compare our proposed approaches with two heuristics.
The first is a heuristic of moving straight towards the closest node with $p_{v}=1$, which we refer to as the \emph{closest terminal} heuristic.
The second is a myopic greedy heuristic, where the rover at any time moves towards the node with the highest $p_v$ among its unvisited neighbors.
We refer to this as the \emph{nearest neighbor} heuristic.
\RC{These are} similar to strategies utilized in the optimal search theory literature \cite{chung2012analysis, bourgault2003optimal}, where myopic strategies with limited lookahead are typically utilized.
Fig. \ref{fig:exp_dist_all_approaches} shows the performance of the best reply, IDAG, nearest neighbor and closest terminal heuristic for various grid sizes ($n$) when \RC{$n_t=1$}.
We generated a $1000$ different probability of success maps, and averaged the expected traveled distance over them to obtain the plotted performance for each $n$.
Also, the error bars in the plot represent the standard deviation of each approach.
In Fig. \ref{fig:exp_dist_all_approaches}, we can see that the best reply and IDAG approach outperform the greedy nearest neighbor heuristic as well as the closest terminal heuristic significantly.
Moreover, the best reply approach outperforms the IDAG approach for larger $n$.

In order to gain more insight into the nature of the solution produced by our proposed approaches, we next consider a scenario where \RC{$n_t=4$}, where we place the four nodes of guaranteed success at the four corners of the workspace, i.e., at $(0,0)$, $(0,n-1)$
$(n-1,0)$ and $(n-1,n-1)$.
Fig. \ref{fig:k_4_path_25_asg} shows the ASG of the best reply process and log-linear learning for a sample such scenario, where we impose a computational time limit of $1$ hour on the log-linear learning approach.
\AM{Best reply process carried out multiple times. How to express this?}
We see that in both cases, the resulting ASG is a forest with $4$ trees, each denoted with a different color in Fig. \ref{fig:k_4_path_25_asg}, where the roots of the $4$ trees correspond to the $4$ nodes of guaranteed success.
As discussed in Section \ref{sec:non_myopic_planners}, the solution ASG of the best reply process is an equilibrium where no node can improve its expected traveled distance by switching the neighbor it routes to.
The route followed by the rover is also plotted on the ASG, which can be seen to visit nodes of higher probability of success.
Fig. \ref{fig:k_4_path_25_plan_conn_path} shows a plot of the routes traveled by the IDAG and the nearest neighbor approach.
In this instance, the paths produced by the best reply and log-linear learning approach were the same as that of the IDAG approach.


\begin{figure}
    \centering
     \includegraphics[width=0.7\linewidth]{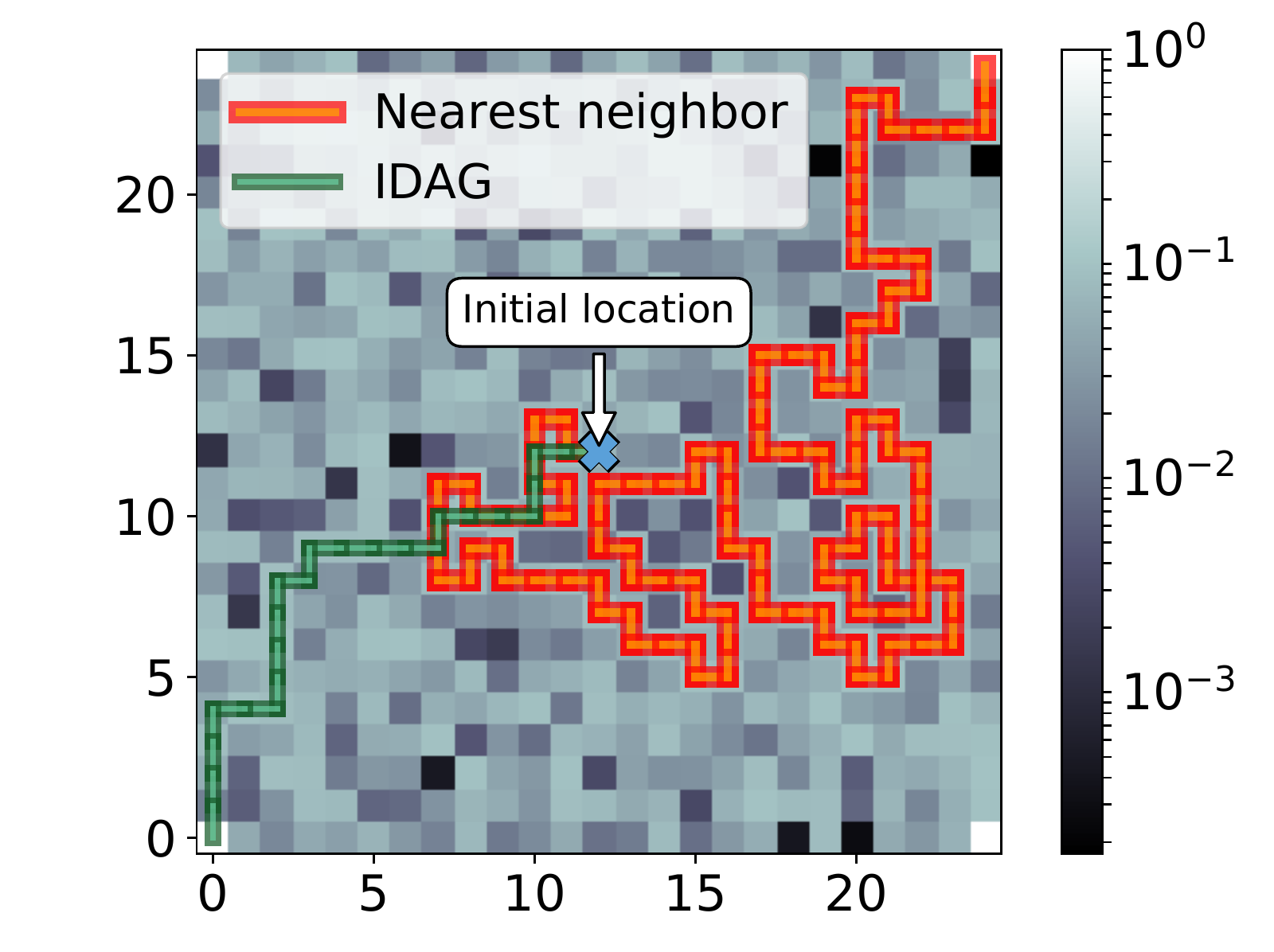}
\caption{\small Path traveled by IDAG and nearest neighbor approach for $n=25$, when there are four nodes with $p_v=1$.
The solution path produced by the best reply and log-linear learning approaches are the same as that of the IDAG approach in this instance.
The background color plot specifies the probability of success of each node.
Readers are referred to the color pdf for better visibility.}
    \label{fig:k_4_path_25_plan_conn_path}
\end{figure}


We next consider the case of \RC{$n_t=0$}, which corresponds to no terminal node being present.
This is an instance of the Min-Exp-Cost-Path-NT problem.
In this setting, the solution we are looking for is a tour of all nodes $\{v \in \mathcal{V}:p_{v}>0\}$ that minimizes $\sum_{e \in \mathcal{E}(\mathcal{P})} \prod_{v \in \mathcal{V}(\mathcal{P}_{e})} (1-p_v) l_{e}$.
In order to facilitate the use of our approaches on the Min-Exp-Cost-Path-NT problem, we introduce a terminal node in the grid graph as discussed in the construction in the proof of Lemma \ref{lemma:mecpntd_red_mecpd}.
We include an edge weight $l = 1.5 \times \frac{D}{\min_{v} p_{v}}$ between the artificial terminal node and all other nodes, where $D=2n$ is the diameter of the graph.
Note that these solution paths may not visit all the nodes in the grid graph, due to the limited computation time.
The best reply process was run $100$ times and the best solution was selected among the solutions produced.
Moreover, we impose a computational time limit of $1$ hour on the log-linear learning approach.
Fig. \ref{fig:tour_25_asg} shows the ASG for the best reply and log-linear learning process as well as the path traveled from the starting node for both cases.
We can see that the paths produced by both approaches traverse through nodes of high probability of success.
Since success is not guaranteed when traversing along a solution path of an approach, expected distance \RC{until} success is no longer well defined.
In other words, we no longer have a single metric by which to judge the quality of a solution.
Instead, we now have two metrics, the probability of failure along a path and the expected distance of traversing the path.
Table. \ref{table:tour_performance} shows the performance of the best reply and log-linear approaches on these metrics for the sample scenario shown in Fig. \ref{fig:tour_25_asg}.
We see that both best reply and log-linear approaches produce a solution with good performance.

\begin{figure}
       \hspace{-0.175in}
        \includegraphics[width=1.1\linewidth]{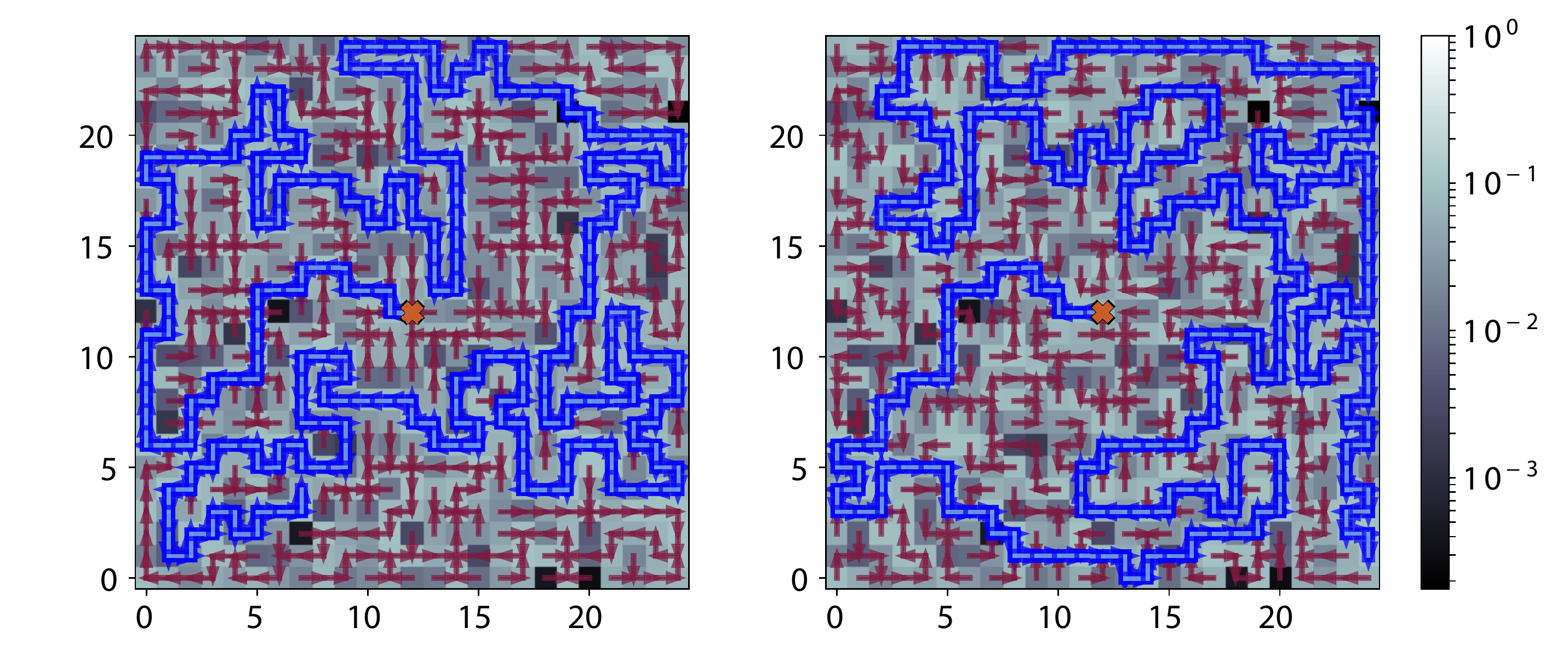}
\caption{\small Acyclic successor graph (ASG) of (left) best reply process and (right) log-linear learning process for $n=25$ when there is no terminal node.
The path traveled from the starting node is also plotted (in blue).
The starting position at $(12,12)$ is marked by the orange ``x''.
The background color plot specifies the probability of success of each node.}
    \label{fig:tour_25_asg}
    \vspace{-0.2in}
\end{figure}


\begin{table}
\centering
\resizebox{0.95\linewidth}{!}{%

\begin{tabular}{c|c|c}
 & Exp distance (m) & Prob of fail along path \\ \hline 
Best reply & $12.40$ & $3.5e-7$  \\ \hline 
Log-linear learning & $12.51$ & $2.5e-7$
\end{tabular}}
\caption{\small The expected traveled  distance and the probability of failure along path for the best reply and log-linear learning approaches.
}
\label{table:tour_performance}
\end{table}

\subsection{Connectivity seeking robot}\label{subsec:conn_seeking_robot}
In this section, we consider the scenario of a robot seeking to get connected to a remote station.
We say that the robot is connected if it is able to reliably transfer information to the remote station.
This would imply satisfying a Quality of Service (QoS) requirement such as a target bit error rate (BER), which would in turn imply a minimum required received channel power given a fixed transmit power.
Thus, in order for the robot to get connected, it needs to find a location where the channel power, when transmitting from that location, would be greater than the minimum required channel power.
However, the robot's prior knowledge of the channel is stochastic. 
Thus, for a robot seeking to do this in an energy efficient manner, its goal would be to plan a path such that it gets connected with a minimum expected traveled distance.

For the robot to plan such a path, it would require an assessment of the channel quality at any unvisited location.
In previous work, we have shown how the robot can probabilistically predict the spatial variations of the channel based on a few a priori measurements \cite{malmirchegini2012spatial}.
Moreover, we consider the multipath component to be time varying as in \cite{muralidharan2018pconn}.
See \cite{malmirchegini2012spatial} for details on this channel prediction as well as performance of this framework with real data and in different environments.

Consider a scenario where the robot is located in the center of a $50$ m $\times$ $50$ m workspace as shown in Fig. \ref{fig:path_plan_conn_path}, with the remote station located at the origin.
The channel is generated using the realistic probabilistic channel model in \cite{goldsmith2005wireless, malmirchegini2012spatial}, with the following parameters that were obtained from real channel measurements in downtown San Francisco \cite{smith2004urban} : path loss exponent $n_{\text{PL}} = 4.2$, shadowing power $\sigma_{\text{SH}} = 2.9$ and shadowing decorrelation distance $\beta_{\text{SH}} = 12.92$ m.
Moreover, the multipath fading is taken to be uncorrelated Rician fading with the parameter $K_{\text{ric}} = 1.59$.
In order for the robot to be connected, we require a minimum required received power of $P_{R,\text{th},\text{dBm}} = -80$ dBmW.
We take the maximum transmission power of a node to be $P_{0,\text{dBm}} = 27$ dBmW \cite{lonn2004output}.

The robot is assumed to have $5$ \% a priori measurements in the workspace.
It utilizes the channel prediction framework described above to predict the channel at any unvisited location.
We discretize the workspace of the robot into cells of size $1$ m by $1$ m.
A cell is connected if there exists a location in the cell that is connected.
Then, the channel prediction framework of \cite{malmirchegini2012spatial} is utilized to estimate the probability of connectivity of a cell.
See \cite{muralidharan2018pconn} for more details on this estimation.
We next construct a grid graph with each cell serving as a node on our graph.
This gives us a grid graph of dimension $50$x$50$ with a probability of connectivity assigned to each node.
We also add a new terminal node to the graph with probability of connectivity $1$, which represents the remote station at the origin.
We attach the node in the workspace closest to the remote station to this terminal node with an edge cost equal to the expected distance \RC{until} connectivity when moving straight towards the remote station from the node.
This can be calculated based on the work in \cite{muralidharan2017fpd}.

\begin{figure}
    \centering
    \includegraphics[width=0.75\linewidth]{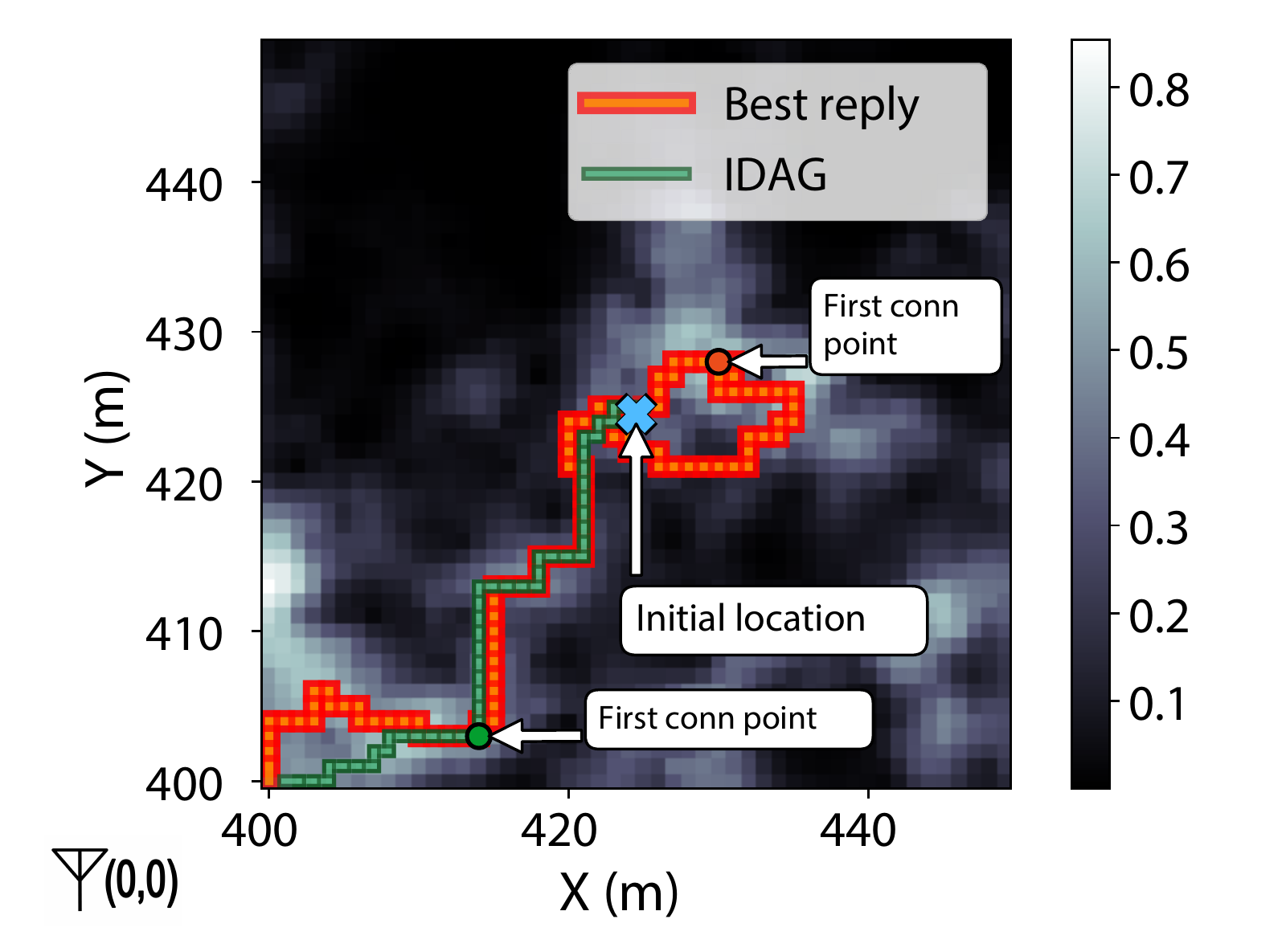}
\caption{\small Solution paths produced by the best reply and IDAG approaches for a channel realization.
Also shown is the first connected node on the respective paths for the true channel realization.
The background plot denotes the predicted probability of connectivity, which is used by the robot for path planning.}
    \label{fig:path_plan_conn_path}
\end{figure}

\begin{table*}[ht]
\centering
\resizebox{0.67\linewidth}{!}{%

\begin{tabular}{c|c|c|c|c}
 & Best reply & IDAG & Nearest neighbor & Closest terminal \\ \hline 
Avg distance (m) & $28.40 \rpm 25.93 $ & $32.90 \rpm 26.12$ & $44.17 \rpm 56.22$ & $50.24 \rpm 30.38$ 
\end{tabular}}
\caption{\small The average traveled distance along with the corresponding standard deviation, for our proposed approaches and for the greedy nearest neighbor and closest terminal heuristic approaches.
The average is obtained by averaging over $500$ channel realizations.
We can see that our approaches results in a significant reduction in the traveled distance.}
\label{table:performance_ch_approach}
\end{table*}

\begin{figure}
    \centering
    \includegraphics[width=0.7\linewidth]{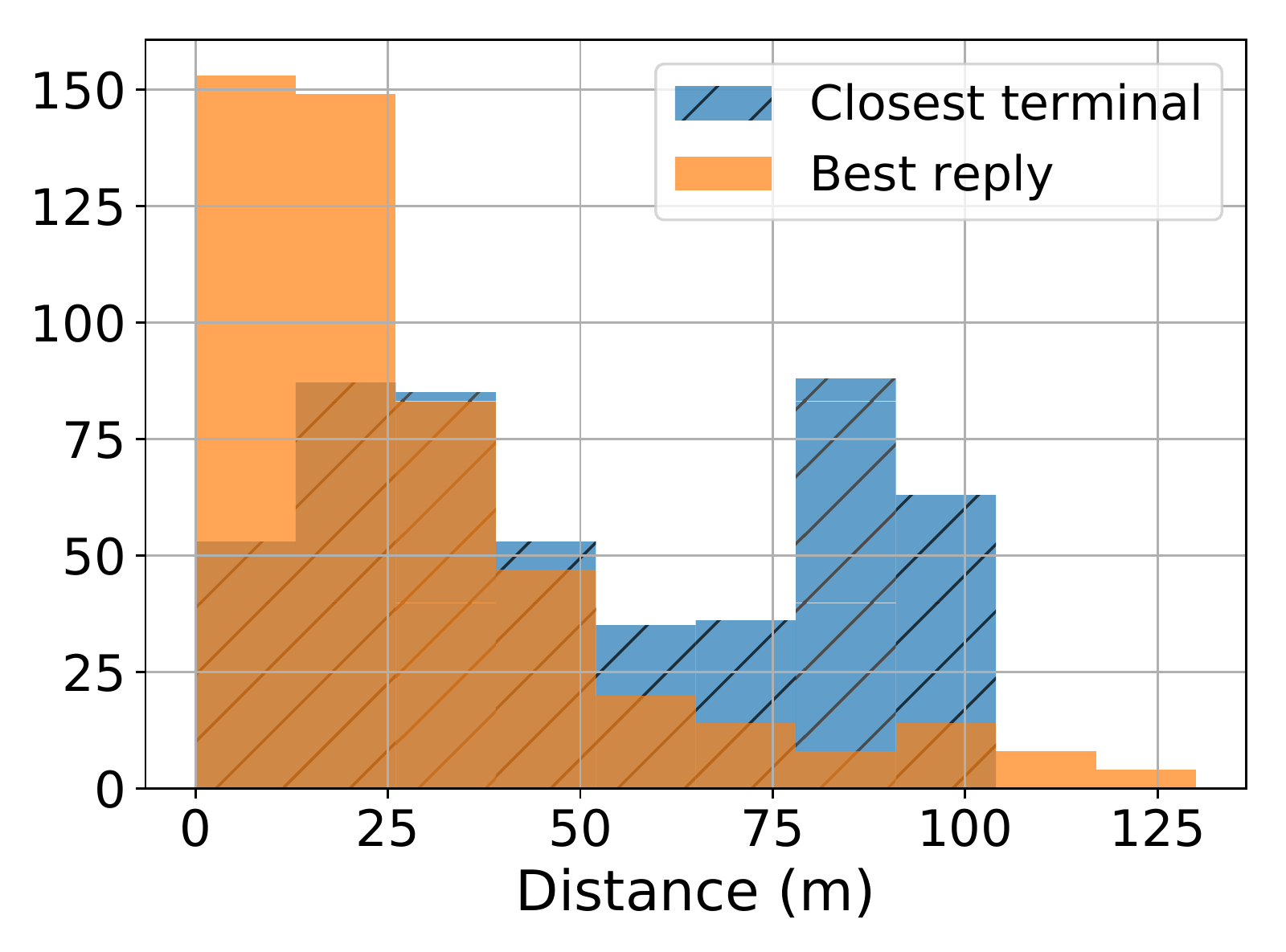}
\caption{\small Histogram of the expected cost of the best reply and closest terminal heuristic over $500$ channel realizations.}
    \vspace{-0.1in}
    \label{fig:hist_approaches}
\end{figure}

We next compare our proposed approaches with the greedy nearest neighbor heuristic as well as the closest terminal heuristic of moving straight towards the remote station.
We calculate the performance of the approaches based on the true probability of connectivity of a node calculated based on the true value of the channel.
Fig. \ref{fig:path_plan_conn_path} shows the solution path produced by the best reply and IDAG heuristic for a sample channel realization.
The background plot denotes the predicted probability of connectivity.
We see that the paths produced take detours on the path to the connected point to visit areas of good probability of connectivity.
Table. \ref{table:performance_ch_approach} shows the expected distance along with the corresponding standard deviation, for the best reply, IDAG, nearest neighbor and closest terminal approaches averaged over $500$ channel realizations.
We do not include the performance of log-linear as it takes longer to arrive at a good solution and is thus impractical to average over $500$ channel realizations.
However, in our simulations, we did observe that the performance of best reply was generally similar to the performance of log-linear learning.
We see that the best reply and IDAG approach outperformed the nearest neighbor and closest terminal heuristics significantly.
For instance, the best reply approach provided an overall $35 \%$ and $44 \%$ reduction in the expected traveled distance when compared to the nearest neighbor and closest terminal heuristics respectively.
Fig. \ref{fig:hist_approaches} shows the histogram of the expected cost of the best reply and closest terminal heuristic over the $500$ channel realizations.
We can see that the expected cost associated with the best reply heuristic is typically better than that associated with the closest terminal heuristic.

\begin{remark}
Note that our framework can be extended to the case where the robot updates the probabilities of success as it operates in the environment.
\end{remark}

\section{Conclusions \RC{and future work}}
In this paper, we considered the problem of path planning on a graph for minimizing the expected cost \RC{until} success.
We showed that this problem is NP-hard and that it can be posed in a Markov Decision Process framework as a stochastic shortest path problem.
We proposed a path planner based on a game-theoretic framework that yields an $\epsilon$-suboptimal solution to this problem asymptotically.
In addition, we also proposed two non-myopic suboptimal strategies that find a good solution efficiently.
Finally, through numerical results we showed that the proposed path planners outperform naive and greedy heuristics significantly.
We considered two scenarios in the simulations, that of a rover on mars searching for an object for scientific study, and that of a realistic path planning scenario for a connectivity seeking robot.
Our results then indicated a significant reduction in the expected traveled distance (e.g., $35 \%$ reduction for the path planning for connectivity scenario), when using our proposed approaches.

\RC{There are several open questions and interesting directions to pursue in this area.
One such direction is developing algorithms with provable performance guarantees that run in polynomial time ($\alpha$-approximation algorithms) for the Min-Exp-Cost-Path problem.
The applicability of the results of this paper to areas such as satisficing search and theorem solving \cite{simon1975optimal} is another interesting future direction.
}

\AM{The paper \cite{molaverdikhani2009mapping} creates a probability map of finding habitable exoplanets in space.
We can combine this with the fact that TSP is used in Astronomy \cite{bailey2000fuel} to motivate applying our approach to this.
Another example is exploration for water on mars.
Based on the orbiter around mars, scientists have hypotheses that water may be found in some valley like creeks (by creating a gravitational map) and have recently confirmed (in 2015) the presence of saltwater in one of those creeks.
Imagine the ground rover searching for water and we have probability map of finding water at those locations.}

\appendix
\section{Appendix}

\subsection{Proof of Lemma \ref{lemma:simple_path_complete_graph}}\label{appendix:lemma_simple_path}

We first describe some properties of the solution of Min-Exp-Cost-Path and Min-Exp-Cost-Simple-Path.

\begin{definition}
Consider a path $\mathcal{P}=(v_1,v_2,\cdots, v_m)$. 
A node $v_i$ is a \emph{revisited} node in the $i^{\text{th}}$ location of $\mathcal{P}$ if $v_i=v_j$ for some $j<i$.
A node $v_i$ is a \emph{first-visit} node in the $i^{\text{th}}$ location of $\mathcal{P}$ if $v_i\neq v_j$ for all $j<i$.
\end{definition}

\begin{property}\label{property:min_exp_cost_path}
Let $\mathcal{P}^{*}=(v_1,v_2,\cdots, v_m)$ be a solution to Min-Exp-Cost-Path on $\mathcal{G}$.
Consider any subpath $(v_{i},v_{i+1},\cdots,v_{j-1},v_{j})$ of $\mathcal{P}^{*}$ such that $v_i$ and $v_j$ are first-visit nodes, and  $v_{i+1},\cdots,v_{j-1}$ are revisited nodes.
Then, \RC{$(v_{i},v_{i+1},\cdots,v_{j-1}, v_{j})$} is the shortest path between $v_i$ and $v_j$.
\end{property}

\begin{proof}
We show this by contradiction.
Assume otherwise, i.e., \RC{$(v_{i},v_{i+1},\cdots,v_{j-1}, v_{j})$} is not the shortest path between $v_{i}$ and $v_{j}$.
Let $\mathcal{Q}$ be the path produced by replacing this subpath in $\mathcal{P}^{*}$ with the shortest path between $v_i$ and $v_j$.
\RC{
Let us denote this shortest path by $(v_i, u_{i+1}, \cdots, u_{\tilde{j}-1}, u_{\tilde{j}})$ where $u_{\tilde{j}} = v_{j}$.
Then,
\begin{align*}
C(\mathcal{Q}, i) & =   (1-p_{v_i})\Bigg[ l_{v_{i}u_{i+1}} + \Big[\prod_{k \in \mathcal{K}_{i+1}}(1-p_{u_k})\Big]l_{u_{i+1}u_{i+2}} \\
& \;\;+ \cdots + \Big[\prod_{k \in \mathcal{K}_{\tilde{j}-1}}(1-p_{u_k})\Big]\Big[l_{u_{\tilde{j}-1}u_{\tilde{j}}} + C(\mathcal{Q},\tilde{j})\Big] \Bigg]\\
& \leq (1-p_{v_i})\Bigg[ l_{v_{i}v_{j}}^{\text{min}} + \Big[\prod_{k \in \mathcal{K}_{\tilde{j}-1}}(1-p_{u_k})\Big]C(\mathcal{Q},\tilde{j})\Bigg],
\end{align*}
where $\mathcal{K}_{m} = \{k\in \{i+1,\cdots,m\}: u_k \text{ is a first visit node of } \mathcal{Q}\}$.
}
The nodes \RC{$(u_{i+1}, \cdots, u_{\tilde{j}})$} could be first visit nodes of $\mathcal{Q}$ or repeated nodes.
We next show that in either scenario the expected cost of $\mathcal{Q}$ would be smaller than that of $\mathcal{P}$.
\RC{
If they are all revisited nodes or if they are first-visit nodes that are not revisited after node $u_{\tilde{j}}$, then $C(\mathcal{Q},\tilde{j}) = C(\mathcal{P}^{*}, j)$.
If some or all of $(u_{i+1}, \cdots, u_{\tilde{j}})$ are first-visit nodes of $\mathcal{Q}$ that are visited later on, then $[\prod_{k \in \mathcal{K}_{\tilde{j}-1}}(1-p_{u_k})]C(\mathcal{Q},\tilde{j}) \leq C(\mathcal{P}^{*}, j)$, since success at a first visit node $u_k$ can occur earlier in  path $\mathcal{Q}$ in comparison to $\mathcal{P}^{*}$ (which discounts the cost of all following edges).
Thus, in either case, we have the inequality
\begin{align*}
C(\mathcal{Q}, i) & \leq (1-p_{v_{i}}) \left[l_{v_{i}v_{j}}^{\text{min}}+C(\mathcal{P}^{*},j)\right]\\
& <  (1-p_{v_i})\left[ l_{v_{i}v_{i+1}} + \cdots + l_{v_{j-1}v_{j}} + C(\mathcal{P}^{*},j) \right]\\
& = C(\mathcal{P}^{*}, i).
\end{align*}
This implies that $C(\mathcal{Q},1) < C(\mathcal{P}^{*},1)$ resulting in a contradiction.
} 
\end{proof}

\begin{property}\label{property:min_exp_cost_simple_path}
Let $\mathcal{P}^{*}=(v_1,v_2,\cdots, v_m)$ be a solution of Min-Exp-Cost-Simple-Path on complete graph $\mathcal{G}_{\text{comp}}$.
Consider any two consecutive nodes $v_i$ and $v_{i+1}$.
The shortest path between $v_i$ and $v_{i+1}$ in $\mathcal{G}$ would only consist of nodes that have been visited earlier in $\mathcal{P}^{*}$.
\end{property}
\AM{This is true only if the nodes on the shortest path have non-zero probability of connectivity.}
\begin{proof}
Suppose this is not true for consecutive nodes $v_i$ and $v_{i+1}$.
Then there exists at least a single node $u$ that lies on the shortest path between $v_i$ and $v_{i+1}$, and that has not been visited earlier in $\mathcal{P}^{*}$.
Let $\mathcal{Q}$ be the path formed from $\mathcal{P}^{*}$ when $u$ is added between $v_{i}$ and $v_{i+1}$. 
The expected cost of $\mathcal{Q}$ from the $i^{\text{th}}$ node onwards is given by 
\begin{align*}
C(\mathcal{Q},i) & = (1-p_{v_{i}}) \left[l_{v_{i}u} + (1-p_{u})\left[l_{uv_{i+1}} + C(\mathcal{Q},i+2)\right]\right] \\
& < (1-p_{v_{i}}) \left[l_{v_{i}v_{i+1}}+C(\mathcal{P}^{*},i+1)\right].
\end{align*}
This implies that the expected cost of $Q$ would be less than that of $\mathcal{P}^{*}$, resulting in a contradiction.
\end{proof}

\begin{proof}[Proof of Lemma \ref{lemma:simple_path_complete_graph}]
Let $\mathcal{P}$ be the solution to Min-Exp-Cost-Path on $\mathcal{G}$ and let $\mathcal{Q}$ be the solution of the Min-Exp-Cost-Simple-Path on $\mathcal{G}_{\text{comp}}$.
From Property \ref{property:min_exp_cost_path}, we know that the path produced by removing revisited nodes in $\mathcal{P}$, will produce a feasible solution to Min-Exp-Cost-Simple-Path on $\mathcal{G}_{\text{comp}}$ with the same cost as $\mathcal{P}$.
Thus, the cost of $\mathcal{P}$ is greater than or equal that of $\mathcal{Q}$.
Similarly, from Property \ref{property:min_exp_cost_simple_path}, we know that the path produced by expanding the shortest path between any adjacent nodes in $\mathcal{Q}$, will be a feasible solution to Min-Exp-Cost-Path on $\mathcal{G}$ with the same cost as $\mathcal{Q}$.
Thus, this path produced from $\mathcal{Q}$ will be an optimal solution to Min-Exp-Cost-Path on $\mathcal{G}$.
\end{proof}

\RC{
\subsection{Proof of Theorem \ref{theorem:log_lin_asympt}}\label{appendix:proof_log_lin_asympt}
Log-linear learning induces a Markov process on the action profile space $A_{\text{ASG}} \cup A_{\emptyset}$, where $A_{\emptyset} = \{\mu: \mu_{v}=a_{\emptyset} \text{ for some } v\}$.
In the following lemma, we first show that $A_{\text{ASG}}$ is a closed communicating recurrent class.
\begin{lemma}\label{lemma:A_ASG_closed_comm_class}
$A_{\text{ASG}}$ is a closed communicating recurrent class.
\end{lemma}
}
\begin{proof}
\RC{
We first show that $A_{\text{ASG}}$ is a communicating class, i.e., there is a finite transition sequence from $\mu^{s}$ to $\mu^{f}$  with non-zero probability for all $\mu^{s}, \mu^{f} \in A_{\text{ASG}}$.
Consider the set of states $R_0,R_1,\cdots,$ defined by the recursion
$R_{k+1} = \{v:\mu_{v}^{f} \in R_{k}\}$, where $ R_0 = T$,
i.e., $R_{k}$ is the set of all nodes that are $k$ hops away from the set of terminal nodes $T$ in the ASG $\mathcal{SG}(\mu^{f})$.
Let $\bar{k}$ be the last of the sets that is non-empty.
Since $\mu^{f} \in A_{\text{ASG}}$, we have $\bar{k} \leq |\mathcal{V}|$ and $\cup_{m=0}^{\bar{k}}R_{m} = \mathcal{V}$.
We transition from $\mu^{s}$ to $\mu^{f}$ by sequentially switching from $\mu_{v} = \mu_{v}^{s}$ to $\mu_v = \mu_{v}^{f}$, for all $v \in R_{k}$, starting at $k=1$ and incrementing $k$ until $k=\bar{k}$, i.e., we first change the action of nodes in $R_1$, and then $R_2$ and so on until $R_k$.
We next show that this transition sequence has a non-zero probability by showing that each component transition has a non-zero probability.
At stage $k+1$, consider the transition where we switch the action of a node $v \in R_k$, and let $\mu$ be the current action.
At this stage we have already changed the action of players in $R_1, \cdots, R_k$, and for the current graph $\mathcal{SG}(\mu)$, there is a path leading from $\mu_{v}^{f} \in R_k$ to a terminal node in $R_0$.
Moreover, $\mu_{v}^{f}$ is not upstream of $v$ since the intermediate nodes of the path are in $R_{k-1}, \cdots, R_{1}$.
Then, $\mu_{v}^{f} \in A_v^{c}(\mu_{-v})$, which implies that the transition  $(\mu_v^{s},\mu_{-v}) \rightarrow (\mu_{v}^{f}, \mu_{-v})$ has a non-zero probability.
Thus, $A_{\text{ASG}}$ is a communicating class.
}

\RC{
We next show that $A_{\text{ASG}}$ is closed.
Consider a state $\mu \in A_{\text{ASG}}$, and a node $v \in \mathcal{V'}$.
Then, $A_v^{c}(\mu_{-v})$ is not empty, since $\mu_v \in A_{v}^{c}(\mu_{-v})$.
This implies that $\mu_v$ can not be set as the null action $a_{\emptyset}$.
Thus, $A_{\text{ASG}}$ is closed.
Since $A_{\text{ASG}}$ is a closed communicating class, every action profile $\mu \in A_{\text{ASG}}$ is a recurrent state.
}
\end{proof}

\RC{
We next show, in the following lemma, that all states in $A_{\emptyset}$ are transient states.
\begin{lemma}\label{lemma:null_set_transience}
Any state $\mu \in A_{\emptyset}$ is a transient state.
\end{lemma}
}
\begin{proof}
\RC{
Consider a state $\mu^{s} \in A_{\emptyset}$ and a state $\mu^{f} \in A_{\text{ASG}}$.
We can design a transition sequence of non-zero probability from $\mu^{s}$ to $\mu^{f}$ similar to how we did so in the proof of Lemma \ref{lemma:A_ASG_closed_comm_class}, as the sequence designed did not depend on $\mu^{s}$.
Moreover, from Lemma \ref{lemma:A_ASG_closed_comm_class}, we know that $A_{\text{ASG}}$ is a closed class.
Thus, there is a finite non-zero probability that the state $\mu^{s} \in A_{\emptyset}$ will never be revisited.
}
\end{proof}

\begin{proof}[Proof of Theorem \ref{theorem:log_lin_asympt}]
\RC{
From Lemma \ref{lemma:A_ASG_closed_comm_class} and Lemma \ref{lemma:null_set_transience}, we know that there is exactly one closed communicating recurrent class.
Thus, the stationary distribution of the Markov chain induced by log-linear learning is unique.
The transition probability from state $\mu$ to $\mu^{'}=(\mu_{v}^{'}, \mu_{-v})$ for $\mu, \mu^{'} \in A_{\text{ASG}}$ is given as 
\begin{align*}
P_{\mu\mu^{'}} = \frac{1}{|\mathcal{V}^{'}|}\frac{e^{-\frac{1}{\tau}\left(\mathcal{J}_v(\mu_v^{'}, \mu_{-v})\right)}}{\sum_{\mu_{v}^{''} \in A_v^{c}(\mu_{-v})}e^{-\frac{1}{\tau}\left(\mathcal{J}_{v}(\mu_v^{''}, \mu_{-v})\right)}},
\end{align*}
denote .
We can reformulate this as 
\begin{align*}
P_{\mu\mu^{'}} = \frac{1}{|\mathcal{V}^{'}|}\frac{e^{-\frac{1}{\tau}\left(\phi(\mu_v^{'}, \mu_{-v})\right)}}{\sum_{\mu_{v}^{''} \in A_v^{c}(\mu_{-v})}e^{-\frac{1}{\tau}\left(\phi(\mu_v^{''}, \mu_{-v}(k))\right)}},
\end{align*}
using $\mathcal{J}_v(\mu_v^{'}, \mu_{-v}) - \mathcal{J}_{v}(\mu_v, \mu_{-v}) = \phi(\mu_v^{'}, \mu_{-v}) - \phi(\mu_v, \mu_{-v})$ from Lemma \ref{lemma:potential_game}.
Then, we can see that the probability distribution $\Pi \in \Delta (A_{\text{ASG}})$ given by
\begin{align*}
\Pi(\mu) = \frac{e^{-\frac{1}{\tau}\phi(\mu)}}{\sum_{\mu^{''} \in A_{\text{ASG}}}e^{-\frac{1}{\tau}\phi(\mu^{''})}},
\end{align*}
satisfies the detailed balance equation 
$\Pi_{\mu}P_{\mu \mu^{'}} = \Pi_{\mu^{'}}P_{\mu^{'} \mu}$.
Thus, $\Pi$ is the unique stationary distribution.
As temperature $\tau \rightarrow 0$, the weight of the stationary distribution will be on the global minimizers of the potential function \cite{marden2012revisiting}.
In other words,
$
\lim_{\tau \rightarrow 0} \sum_{\mu \in \argmin_{\mu^{'} \in A_{\text{ASG}}}\phi(\mu^{'})} \Pi(\mu) = 1.
$
Thus, asymptotically, log-linear learning provides us with the global minimizer of $\phi(\mu)=  C_{v_{s}}(\mu) + \epsilon^{'}\sum_{v \neq v_s}C_{v}(\mu)$, an $\epsilon$-suboptimal solution to the Min-Exp-Cost-Simple-Path problem.}
\end{proof}

\subsection{Proof of Theorem \ref{theorem:BR_finite_iters}}\label{appendix:proof_br_conv}
\begin{proof}
\RC{
We first show that there exists an $k_l$ such that $\mu(k) \in A_{\text{ASG}}$ for all  $k \geq k_l$.
Let $A_{\emptyset} = \{\mu: \mu_{v}=a_{\emptyset} \text{ for some } v\}$ denote the set of action profiles with at least one player playing a null action.
Consider a action profile $\mu \in A_{\emptyset}$.
Then there must exist a node $v \in \{u: \mu_{u} = a_{\emptyset}\}$ which has a neighbor in $T \cup \{u: \mu_{u} \neq a_{\emptyset}\}$, since otherwise $\{u:\mu_{u} = a_{\emptyset}\}$ and $T \cup \{u: \mu_{u} \neq a_{\emptyset}\}$ are not connected, contradicting the assumption that the graph is connected.
Then, $A_{v}^{c}(\mu_{-v})$ is non-empty, and when node $v$ is selected in the round robin iteration it will play a non-null action.
Moreover, $\mu_{v}\neq a_{\emptyset}$ for all subsequent iterations, since its current action at any iteration $k$ will always belong to $A_{v}^{c}(\mu_{-v}(k))$.
We can apply this reasoning repeatedly to show that eventually at some iteration $k_l$ the set $\{u: \mu_{u}(k_l) = a_{\emptyset}\}$ will be empty, i.e., $\mu(k_l) \in A_{\text{ASG}}$.
Furthermore, $\mu(k) \in A_{\text{ASG}}$ for all $k \geq k_l$.
}

\RC{We next}
prove that $C_{v}(\mu(k+1)) \leq C_{v}(\mu(k))$ for all $v \in \mathcal{V}'$ and for all $k$.
Let $v$ be the node selected at stage $k+1$.
\RC{
Clearly, if $\mu_{v}(k) = a_{\emptyset}$ this is true.
Else,}
\begin{align}
C_{v}(\mu(k+1)) & = \min_{\RC{u \in A_v^{c}(\mu_{-v}(k))}} \left\{ (1-p_v)\left[l_{vu} + C_{u}(\mu(k))\right] \right\} \nonumber\\
& \leq (1-p_v)\left[l_{v\mu_{v}(k)} + C_{\mu_{v}(k)}(\mu(k))\right] \label{eq:1}\\
& = C_{v}(\mu(k))\nonumber, 
\end{align}
where (\ref{eq:1}) follows since $\mu_{v}(k) \in A_v^{c}(\mu_{-v}(k))$.
From (\ref{eq:rec_exp_dist}), we have that 
$ C_u(\mu(k+1)) \leq C_u(\mu(k)), \;\; \forall u \in U_{v}(\mu)$,
where $U_{v}(\mu)$ is the set of upstream nodes from $v$.
Furthermore,
$ C_u(\mu(k+1)) = C_u(\mu(k)), \;\; \forall u \notin U_{v}(\mu)$.
Thus, $C_{v}(\mu(k+1)) \leq C_{v}(\mu(k))$ for all $v \in \mathcal{V}'$.

Since $\{C_{v}(\mu(k))\}_{k}$ is a monotonically non-increasing sequence, bounded by below from $0$, we know that the limit exists.
Moreover, since $\mu$ belongs to a finite space, we know that convergence must occur in a finite number of iterations.
It should be noted however, that the limit can be different based on the order \RC{of the nodes in the round robin}.
Let \RC{$\mu^{*} \in A_{\text{ASG}}$} denote the solution at convergence for the particular order of nodes.
\RC{We assume that, when selecting $\mu_{v}$, ties are broken using a consistent set of rules, since otherwise we may cycle repeatedly through action profiles having the same expected costs $\{C_{v}(\mu)\}_{v}$.}

We next show that we converge to this limit in $|\mathcal{V}'|^2$ iterations.
Let $n=|\mathcal{V}'|$.
Consider the set of states $R_0,R_1,\cdots,$ defined by the recursion
$R_{k+1} = \{v:\mu_{v}^{*} \in R_{k}\}$, where $ R_0 = T$,
i.e., $R_{k}$ is the set of all nodes that are $k$ hops away from the set of terminal nodes $T$ in the ASG $\mathcal{SG}(\mu^{*})$.
Let $\bar{k}$ be the last of the sets that is non-empty.
Since \RC{$\mu^{*} \in A_{\text{ASG}}$}, we have $\bar{k} \leq n$ and $\cup_{m=0}^{\bar{k}}R_{m} = \mathcal{V}$.
We next show by induction that
$ \mu_v(nk) = \mu_v^{*}, \;\; \forall v \in \cup_{m=0}^{k}R_m$,
for $k=0,1,\cdots,\bar{k}$.
This is true for $k=0$.
Assume that it holds true at stage $k$, i.e., $\mu_{v}(nk) = \mu_{v}^{*}$ for all $v \in\cup_{m=1}^{k}R_{0}$.
Since $\{C_{v}(\mu(k))\}_{k}$ is monotonically non-increasing, we have $C_{v}(\mu^{*}) \leq C_{v}(\mu(k+1))$.
Moreover, since any node $v \in \cup_{m=0}^{k+1}R_m$ would be selected once in \RC{round} $k+1$ of the \RC{round robin} process, we have
\begin{align}
C_{v}(\mu(n(k+1))) &= \min_{\RC{u \in A_{v}^{c}(\mu_{-v}(n(k+1)-1))}} \Big\{ (1-p_v) \times   \nonumber \\
&  \;\;\;\;\;\;\;\;\;\;\;\;\; \big[l_{vu} + C_{u}(\mu(n(k+1)-1))\big] \Big\} \nonumber\\
& \leq (1-p_v)\left[l_{v\mu_{v}^{*}} + C_{\mu_v}(\mu^{*})\right] \label{eq:2} \\ 
& = C_{v}(\mu^{*}). \nonumber
\end{align}
where (\ref{eq:2}) follows based on the induction hypothesis, since $\mu_{v}^{*}$ leads to a direct path to a terminal node, and is not an upstream node of $v$.
Thus, $\mu_{v}(n(k+1)) = \mu_{v}^{*}$ for all $v \in \cup_{m=0}^{k+1}R_m$.
This implies that the best reply process, when we cycle through the nodes \RC{in a round robin}, converges within at most $n^2$ iterations.
\end{proof}

\vspace{-0.1in}
\subsection{Relation to \RC{the Discounted-Reward Traveling Salesman Problem}}\label{appendix:pc_tsp_relation}
In this section, we show \RC{the relationship between the Min-Exp-Cost-Path-NT problem of Section \ref{subsec:min_exp_cost_tour} and the Discounted-Reward-TSP, a path planning problem studied in the theoretical computer science community \cite{blum2007approximation}.
Note that this section is merely pointing out the relationship between the objectives/constraints of the two problems, and is not claiming that one is reducible to the other.}
In \RC{Discounted-Reward-TSP}, each node $v$ has a prize $\pi_{v}$ associated with it and each edge $(u, v)$ has a cost $l_{uv}$ associated with it.
The goal is to find a path $\mathcal{P}$ that visits all nodes and that maximizes the discounted reward collected $\sum_{v} \gamma^{l^{\mathcal{P}}_{v}} \pi_{v}$, where $\gamma<1$ is the discount factor, and $l^{\mathcal{P}}_{v} = \sum_{e \in \mathcal{E}(\mathcal{P}_{v})}l_e$ is the cost incurred along path $\mathcal{P}$ until node $v$.

\RC{In the setting of our Min-Exp-Cost-Path-NT problem}, the prize of a node $v$ \RC{is taken as} $\pi_{v} = \log_{\gamma}(1-p_{v})$ for a value of $\gamma < 1$.
Our Min-Exp-Cost-Path-NT objective can \RC{then} be reformulated as $\sum_{e \in \mathcal{E}(\mathcal{P})} \gamma ^{\pi^{\mathcal{P}}_{e}} l_{e}$, where $\pi^{\mathcal{P}}_{e} = \sum_{v \in \mathcal{V}(\mathcal{P}_{e})} \pi_{v}$ is the reward collected along path $\mathcal{P}$ until edge $e$ is encountered.
We can refer to this problem as the Discounted-Cost-TSP problem, drawing a parallel to the Discounted-Reward-TSP problem described above.
\RC{However, note that our problem is not the same as the Discounted-Reward-TSP problem.
Rather, we simply illustrated a relationship between the two problems, which can lead to further future explorations in this area.}

\vspace{-0.1in}
\subsection{Formulation as Stochastic Shortest Path Problem with Recourse}\label{appendix:ssp_recourse}
In this section, we show that we can formulate the Min-Exp-Cost-Path problem as a special case of the stochastic shortest path problem with recourse \cite{polychronopoulos1993stochastic}.
The terminology of stochastic shortest path here is different from its usage in \ref{subsec:mdp_formulation}.
The stochastic shortest path problem with recourse consists of a graph where the edge weights are random variables taking values from a finite range.
As the graph is traversed, the realizations of the cost of an edge is learned when one of its end nodes are visited.
The goal is to find a policy that minimizes the expected cost from a source node $v_s$ to a destination node $v_{t}$.
The best policy would determine where to go next based on the currently available  
information.

Consider the Min-Exp-Cost-Path problem on a graph $\mathcal{G}=(\mathcal{V},\mathcal{E})$, with probability of success $p_{v}\in [0,1]$, for all $v \in \mathcal{V}$.
We can formulate this as a special case of this stochastic shortest path problem with recourse, by adding a node $v_t$ which acts as the destination node.
Each node in $\mathcal{G}$ is connected to $v_t$ with a edge of random weight.
The edge from $v$ to $v_t$ has weight 
$
l_{vv_{t}} =  \left\{\begin{array}{ll} 0, & \text{w.p. } p_{v}\\
\infty, & \text{w.p. } 1-p_{v} \end{array}\right. .
$
The remaining set of edges $\mathcal{E}$ are deterministic.
The solution to the shortest path problem from $v_s$ to $v_t$ with recourse, would provide a policy that would give us the solution to the Min-Exp-Cost-Path problem.
The policy in this special case would produce a path from $v_{s}$ to a node in the set of terminal nodes $T$.
However, the general stochastic shortest path with recourse is a much harder problem to solve than the Min-Exp-Cost-Path problem and
the heuristics utilized for stochastic shortest path with recourse are not particularly suited to our specific problem.
For instance, in the open loop feedback certainty equivalent heuristic \cite{gao2006optimal}, at each iteration, the uncertain edge costs are replaced with their expectation and the next node is chosen according to the deterministic shortest path to the destination.
In our setting this would correspond to the heuristic of moving along the deterministic shortest path to the closest terminal node.
Such a heuristic would ignore the probability of success $p_{v}$ of the nodes.

\bibliography{ref}

\begin{thebibliography}{10}

\bibitem{muralidharan2018pconn}
A.~Muralidharan and Y.~Mostofi, ``Path planning for a connectivity seeking
  robot,'' in {\em Globecom Workshops}, IEEE, 2017.

\bibitem{malmirchegini2012spatial}
M.~Malmirchegini and Y.~Mostofi, ``On the spatial predictability of
  communication channels,'' {\em IEEE Transactions on Wireless Communications},
  vol.~11, no.~3, pp.~964--978, 2012.

\bibitem{molaverdikhani2009mapping}
K.~Molaverdikhani and M.~Tabeshian, ``Mapping the probability of microlensing
  detection of extra-solar planets,'' {\em arXiv preprint arXiv:0911.4424},
  2009.

\bibitem{rosenthal2012someone}
S.~Rosenthal, M.~Veloso, and A.~Dey, ``Is someone in this office available to
  help me?,'' {\em Journal of Intelligent \& Robotic Systems}, vol.~66, no.~1,
  pp.~205--221, 2012.

\bibitem{karaman2010incremental}
S.~Karaman and E.~Frazzoli, ``Incremental sampling-based algorithms for optimal
  motion planning,'' {\em Robotics Science and Systems VI}, vol.~104, 2010.

\bibitem{likhachev2008anytime}
M.~Likhachev, D.~Ferguson, G.~Gordon, A.~Stentz, and S.~Thrun, ``Anytime search
  in dynamic graphs,'' {\em Artificial Intelligence}, vol.~172, no.~14,
  pp.~1613--1643, 2008.

\bibitem{jaillet1985probabilistic}
P.~Jaillet, {\em Probabilistic traveling salesman problems}.
\newblock PhD thesis, Massachusetts Institute of Technology, 1985.

\bibitem{bertsimas1992vehicle}
D.~J. Bertsimas, ``A vehicle routing problem with stochastic demand,'' {\em
  Operations Research}, vol.~40, no.~3, pp.~574--585, 1992.

\bibitem{chung2012analysis}
T.~Chung and J.~Burdick, ``Analysis of search decision making using
  probabilistic search strategies,'' {\em IEEE Transactions on Robotics},
  vol.~28, no.~1, pp.~132--144, 2012.

\bibitem{hollinger2009efficient}
G.~Hollinger, S.~Singh, J.~Djugash, and A.~Kehagias, ``Efficient multi-robot
  search for a moving target,'' {\em The International Journal of Robotics
  Research}, vol.~28, no.~2, pp.~201--219, 2009.

\bibitem{chung2011search}
T.~Chung, G.~Hollinger, and V.~Isler, ``Search and pursuit-evasion in mobile
  robotics,'' {\em Autonomous robots}, vol.~31, no.~4, p.~299, 2011.

\bibitem{bourgault2003optimal}
F.~Bourgault, T.~Furukawa, and H.~Durrant-Whyte, ``Optimal search for a lost
  target in a bayesian world,'' in {\em Field and service robotics},
  pp.~209--222, Springer, 2003.

\bibitem{blum1994minimum}
A.~Blum, P.~Chalasani, D.~Coppersmith, B.~Pulleyblank, P.~Raghavan, and
  M.~Sudan, ``The minimum latency problem,'' in {\em Proceedings of the
  twenty-sixth annual ACM symposium on Theory of computing}, pp.~163--171, ACM,
  1994.

\bibitem{simon1975optimal}
H.~A. Simon and J.~Kadane, ``Optimal problem-solving search: All-or-none
  solutions,'' {\em Artificial Intelligence}, vol.~6, no.~3, pp.~235--247,
  1975.

\bibitem{bertsekas1995dynamic}
D.~Bertsekas, {\em Dynamic programming and optimal control}, vol.~1.
\newblock Athena Scientific Belmont, MA, 1995.

\bibitem{garey2002computers}
M.~Garey and D.~Johnson, {\em Computers and intractability}, vol.~29.
\newblock W. H. Freeman, New York, 2002.

\bibitem{garcia1993loop}
J.~Garcia-Lunes-Aceves, ``Loop-free routing using diffusing computations,''
  {\em IEEE/ACM Transactions on Networking (TON)}, vol.~1, no.~1, pp.~130--141,
  1993.

\bibitem{monderer1996potential}
D.~Monderer and L.~Shapley, ``Potential games,'' {\em Games and economic
  behavior}, vol.~14, no.~1, pp.~124--143, 1996.

\bibitem{marden2012revisiting}
J.~R. Marden and J.~S. Shamma, ``Revisiting log-linear learning: Asynchrony,
  completeness and payoff-based implementation,'' {\em Games and Economic
  Behavior}, vol.~75, no.~2, pp.~788--808, 2012.

\bibitem{fudenberg1991game}
D.~Fudenberg and J.~Tirole, ``Game theory, 1991,'' {\em Cambridge,
  Massachusetts}, vol.~393, p.~12, 1991.

\bibitem{smith2004heuristic}
T.~Smith and R.~Simmons, ``Heuristic search value iteration for pomdps,'' in
  {\em Proceedings of the 20th conference on Uncertainty in artificial
  intelligence}, pp.~520--527, AUAI Press, 2004.

\bibitem{barto1995learning}
A.~Barto, S.~Bradtke, and S.~Singh, ``Learning to act using real-time dynamic
  programming,'' {\em Artificial intelligence}, vol.~72, no.~1-2, pp.~81--138,
  1995.

\bibitem{kirkpatrick1983optimization}
S.~Kirkpatrick, C.~Gelatt, and M.~Vecchi, ``Optimization by simulated
  annealing,'' {\em science}, vol.~220, no.~4598, pp.~671--680, 1983.

\bibitem{goldsmith2005wireless}
A.~Goldsmith, {\em Wireless communications}.
\newblock Cambridge university press, 2005.

\bibitem{smith2004urban}
W.~Smith and D.~Cox, ``Urban propagation modeling for wireless systems,'' tech.
  rep., DTIC Document, 2004.

\bibitem{lonn2004output}
S.~L{\"o}nn, U.~Forssen, P.~Vecchia, A.~Ahlbom, and M.~Feychting, ``Output
  power levels from mobile phones in different geographical areas; implications
  for exposure assessment,'' {\em Occupational and Environmental Medicine},
  vol.~61, no.~9, pp.~769--772, 2004.

\bibitem{muralidharan2017fpd}
A.~Muralidharan and Y.~Mostofi, ``First passage distance to connectivity for
  mobile robots,'' in {\em American Control Conference (ACC)}, pp.~1517--1523,
  2017.

\bibitem{blum2007approximation}
A.~Blum, S.~Chawla, D.~Karger, T.~Lane, A.~Meyerson, and M.~Minkoff,
  ``Approximation algorithms for orienteering and discounted-reward tsp,'' {\em
  SIAM Journal on Computing}, vol.~37, no.~2, pp.~653--670, 2007.

\bibitem{polychronopoulos1993stochastic}
G.~Polychronopoulos and J.~Tsitsiklis, ``Stochastic shortest path problems with
  recourse,'' 1993.

\bibitem{gao2006optimal}
S.~Gao and I.~Chabini, ``Optimal routing policy problems in stochastic
  time-dependent networks,'' {\em Transportation Research Part B:
  Methodological}, vol.~40, no.~2, pp.~93--122, 2006.

\end{thebibliography}
\bibliographystyle{ieeetr}
\end{document}